\documentclass[11pt]{article}

\usepackage[letterpaper,margin=1in]{geometry}

\usepackage{amsmath}
\usepackage{amssymb}
\usepackage{amsthm}

\usepackage{graphicx}
\usepackage{booktabs}
\usepackage{placeins}

\usepackage{enumitem}

\usepackage[authoryear,round]{natbib}
\usepackage{xcolor}
\usepackage[
    colorlinks=true,
    linkcolor=blue,
    citecolor=blue,
    urlcolor=blue
]{hyperref}
\usepackage[capitalise,noabbrev]{cleveref}

\graphicspath{{saddle-escape/figures/}}

\newtheorem{theorem}{Theorem}[section]
\newtheorem{lemma}[theorem]{Lemma}
\newtheorem{proposition}[theorem]{Proposition}
\newtheorem{corollary}[theorem]{Corollary}

\newtheorem{assumption}[theorem]{Assumption}

\newcommand{\R}{\mathbb{R}}
\newcommand{\E}{\mathbb{E}}

\title{A Theory of Saddle Escape in Deep Nonlinear Networks}

\author{
Divit Rawal\thanks{
Department of Physics, University of California, Berkeley.
Email: \texttt{divit.rawal@berkeley.edu}.
}
\and
Michael R. DeWeese\thanks{
Department of Physics, Department of Neuroscience, and
Redwood Center for Theoretical Neuroscience,
University of California, Berkeley.
Email: \texttt{deweese@berkeley.edu}.
}
}

\date{\today}

\begin{document}

\maketitle

\begin{abstract}
In deep networks with small initialization, training can exhibit long plateaus separated by sharp feature-acquisition transitions. Whereas shallow nonlinear networks and deep linear networks are well studied, extending these analyses to deep nonlinear networks remains challenging. We derive exact scalar and matrix identities for the imbalance of layer weight norms, holding for any smooth activation and any differentiable loss, and use the resulting approximate balance law to control the full finite-width gradient flow through its first escape. This gives a scaling law in the escape time: with a prefix of $r$ layers initialized at the bottleneck scale $\varepsilon$, the escape time is $\Theta\left(\Gamma^{-1}\right)$ for $r=1$, $\Theta\left(\Gamma^{-1}\log\left(1/\varepsilon\right)\right)$ for $r=2$, and $\Theta\left(\Gamma^{-1}\varepsilon^{-(r-2)}\right)$ for $r\geq3$, where $\Gamma=\left\|\E[yx]\right\|$. We further show that a multimode teacher leaves the saddle once rather than in stages: with scalar output, all mode overlaps grow through a common amplitude, with relative weights fixed in advance. We find close agreement between our theory and experiment.

\end{abstract}


\section{Introduction}

Deep neural networks do not generally learn at a uniform rate. Training trajectories can sit near an apparently stationary predictor for long stretches, then change rapidly, acquiring different components of a target on widely separated timescales \citep{saxe2014exactsolutionsnonlineardynamics,KuninEtAl2025,SimonEtAl2023,AtanasovEtAl2025,ZhangSaxeLatham2026}. Because training terminates at a finite compute budget, whatever has not yet emerged by that point is absent from the final model, so initialization, depth and data covariance determine not only whether training converges but which components the network ultimately represents.

This staged learning is well studied in the theory of network dynamics. Statistical-mechanical analyses of online learning in teacher-student networks exposed specialization plateaus among hidden units \citep{BiehlSchwarze1995,SaadSolla1995}, and geometric work subsequently traced families of plateaus to singular critical sets created by redundant hidden units \citep{FukumizuAmari2000,WeiEtAl2008}, showing that a plateau's duration and visibility can depend strongly on the data covariance \citep{YoshidaOkada2019}. A particularly interesting regime is small initialization, which produces much of the incremental, low-rank and staged learning behavior documented in deep networks \citep{AroraCohenHazan2018,GissinShalevShwartzDaniely2020}. A deep network initialized this way begins near zero output, and the gradient it receives is attenuated further by the product of small layer gains across depth. The resulting delay can dominate the early trajectory even when the target is already correlated with the inputs. Gradient flow does eventually escape the saddle, but only after this initial waiting period. This delay is invisible to standard convergence statements, and Hessian descriptions can be blind to it, since the origin becomes a degenerate, high-order saddle rather than the strict saddle that appears at $O(1)$ initialization. An important question then remains: \emph{How long does this initial plateau last, and what determines its duration?}

Exact analyses have worked to answer this question, mostly for linear or diagonal architectures. Deep linear networks exhibit long plateaus followed by sharp transitions, with the singular structure of the input-output map setting the order and timing of learning \citep{saxe2014exactsolutionsnonlineardynamics,SaxeMcClellandGanguli2019}; for uniformly small initialization, the first departure from the origin occurs on a logarithmic timescale at depth two and a power-law timescale at greater depth \citep{JacotEtAl2021Saddle}. Depth can also accelerate optimization, strengthen incremental learning and induce low-rank trajectories in linear factorizations \citep{AroraCohenHazan2018,AroraEtAl2019,GissinShalevShwartzDaniely2020}, while diagonal linear models activate coordinates sequentially along saddle-to-saddle paths in the small-initialization limit \citep{Berthier2023,PesmeFlammarion2023}. These results give a precise account of the phenomenon, but rely strongly on the additional structure of linear or diagonal architectures.

The same staging appears, less exactly, across nonlinear and attention-based models. Gradient training favors some functions earlier than others \citep{RahamanEtAl2019}, and stepwise loss curves have been observed in self-supervised Residual Networks (ResNets), supervised convolutional networks, and vision transformers as initialization or feature-learning scale varies \citep{SimonEtAl2023,AtanasovEtAl2025}. Structured targets induce saddle-to-saddle learning in two-layer networks \citep{AbbeBoixAdseraMisiakiewicz2023}, while stationary-point and escape-direction analyses describe the first departures of shallow and deep ReLU networks \citep{WuSimsekGed2025,BantzisSimonJacot2026}. Alternating Gradient Flows separates small-initialization training in two-layer networks into slow feature alignment and rapid norm growth, predicting the order, timing, and magnitude of successive activations \citep{KuninEtAl2025}, and a complementary framework explains staged simplicity bias across linear, ReLU, convolutional and attention architectures through motion between invariant manifolds \citep{ZhangSaxeLatham2026}. Related abrupt or near-stationary stages also appear in linear attention and in-context learning \citep{ZhangEtAl2025,VarreYuceFlammarion2025}, while continuous-depth self-attention exhibits a dependence of representation dynamics on the data and weight spectrum \citep{GeshkovskiEtAl2023}. What remains unclear is the quantitative timescale of the first escape in a deep nonlinear network, and in particular how that timescale depends on layerwise initialization.

The generic theories one might expect to answer this question apply in rather different regimes. Random initialization avoids strict saddles almost surely \citep{LeeEtAl2016}, but unperturbed gradient descent can still take exponential time to leave a sequence of saddle regions \citep{DuEtAl2017}, 
whereas
perturbations or stochastic gradients yield efficient escape under negative-curvature and noise conditions \citep{GeEtAl2015,JinEtAl2017,DaneshmandEtAl2018,CriscitielloBoumal2019}. Here, however, the origin is a degenerate, high-order saddle whose Hessian need not have a negative eigenvalue. Fixed-kernel descriptions are also insufficient: in the neural tangent (or lazy-training) limit, wide networks remain close to their parameter linearization and evolve through an approximately fixed kernel \citep{JacotGabrielHongler2018,LeeEtAl2019,ChizatOyallonBach2019}. Initialization scale and parametrization are known to separate lazy from feature-learning regimes \citep{WoodworthEtAl2020,YangHu2021,KuninEtAl2024}, with mean-field limits providing another description through the evolution of a parameter distribution \citep{MeiMontanariNguyen2018}. During the escape we study, the activations remain in their locally linear range, but the layers with small initialization move by an amount comparable to their own initialization size.

To analyze this ``saddle escape,'' we take inspiration from symmetry and conservation laws. For positively homogeneous activations, adjacent-layer imbalance is exactly conserved by gradient flow \citep{DuHuLee2018}, and more general architectural symmetries constrain gradients, Hessians, and gradient-flow trajectories through conservation laws \citep{KuninEtAl2021,MarcotteGribonvalPeyre2023} that shape critical manifolds, connected components, and reachable regions of parameter space \citep{SimsekEtAl2021,NurissoLeroyVaccarino2024,NurissoEtAl2026}. These invariant manifolds can also organize staged learning or trap later adaptation \citep{ZhangSaxeLatham2026,JoudakiEtAl2026}. Smooth activations are generally not positively homogeneous, so this conservation law does not hold exactly. We instead derive the exact defect $z\sigma'(z)-\sigma(z)$, where the prime indicates a derivative with respect to a function's argument, by which layer imbalance departs from conservation and show that it is high order near the origin. This gives an approximate balance law that persists throughout the first escape and allows us to reduce the full finite-width gradient flow to the growth of a single signal channel.

Using this reduction, we quantitatively characterize the duration of the first plateau in deep nonlinear neural networks with smooth activations satisfying $\sigma(0)=0$ and $\sigma'(0)>0$, trained by gradient flow on squared loss. Suppose the first $r$ layers are initialized at scale $\varepsilon$ while the remaining layers are initialized at scale one; we call $r$ the \emph{critical depth}. For a fixed small escape threshold $\vartheta$, the escape time satisfies
\begin{equation*}
	\tau_\vartheta=\begin{cases}
		\Theta\left(\Gamma^{-1}\right), & r=1,\\
		\Theta\left(\Gamma^{-1}\log\left(1/\varepsilon\right)\right), & r=2,\\
		\Theta\left(\Gamma^{-1}\varepsilon^{-(r-2)}\right), & r\geq3,
	\end{cases}
	\qquad
	\Gamma\doteq\left\|\E[yx]\right\|.
\end{equation*}
Thus the initialization exponent is set by the number of bottleneck layers $r$, not the total depth $L$, which affects only the multiplicative constants. The logarithmic and power-law delays of deep linear networks therefore persist for smooth nonlinear networks, but with the total depth replaced by the number of layers that actually suppress the signal.

The remaining timescale is set by the signal $\Gamma=\|\E[yx]\|$. This reduces the relevant input-distribution dependence documented in teacher-student analyses \citep{GoldtEtAl2020,RefinettiIngrossoGoldt2023} to a single vector moment. For a multimode teacher, $\Gamma$ is the norm of a vector sum of covariance-transformed mode contributions, so input anisotropy allows different modes to reinforce or cancel one another even when their individual magnitudes are fixed. The same identity extends the escape law from Gaussian population models to elliptical inputs and finite datasets, where the signal is measured directly by $\widehat{\Gamma}=\|n^{-1}\sum_i y_i x_i\|$. Our theory requires $\Gamma>0$: the target must present a nonzero first-order correlation to the network at the origin; this is the Gaussian single-index information-exponent-one regime of \citet{BenArousGheissariJagannath2021}.

Finally, the same reduction describes what the network learns during this first escape. Although a multimode teacher may contain several nonlinear features, with scalar output they do not generate separate leading-order departures from the saddle. Instead, the linearized predictor grows along the single presented direction $\Sigma w_\star$, and the individual mode overlaps grow through a common amplitude with relative weights fixed by the teacher and input covariance. Thus the first escape is a joint one: the several teacher modes contribute to one growing signal channel rather than producing a sequence of modewise escapes.

The remainder of this work is organized as follows. Section~\ref{sec:setup} gives the model and data assumptions. Section~\ref{sec:homogeneity-defect} derives the homogeneity-defect identities, Section~\ref{sec:critical-depth} proves the escape law and develops the signal observable, and Section~\ref{sec:multi-mode} establishes alignment and the multimode overlap ratios. Finally, Section~\ref{sec:discussion} summarizes the scope and limitations and discusses several remaining questions.

\section{Setup}\label{sec:setup}

We study $L$-layer fully connected feedforward networks with linear readout:
\begin{equation*}
	f(x;W)=W_L\sigma\left(W_{L-1}\sigma\left(\cdots\sigma\left(W_1x\right)\right)\right).
\end{equation*}
Here we have weight matrices $W_\ell\in\R^{n_\ell\times n_{\ell-1}}$, input dimension $n_0=d$, hidden widths $n_1=\cdots=n_{L-1}=N$, and scalar output $n_L=1$. We set $h_0=x$, write $z_\ell=W_\ell h_{\ell-1}$, and $h_\ell=\sigma\left(z_\ell\right)$ for $1\leq\ell<L$ (we will sometimes refer to the $z_\ell$ and $h_\ell$ as the pre- and post-activations, respectively), and take $f=z_L=W_Lh_{L-1}$; the activation is applied entrywise at the hidden layers only.

We take the activation $\sigma:\R \to \R$ as globally $C^1$ with bounded derivative, $C^{q+1}$ on a neighborhood of the origin, with  $\sigma(0)=0$, and assume it satisfies
\begin{equation*}
	\sigma(u) = \alpha u + a_qu^q + O(u^{q+1}), \qquad \alpha = \sigma'(0) > 0, \quad a_q \neq 0,
\end{equation*}
where $q \geq 2$ is the order of the first nonlinear term. We require the local regularity conditions for the Taylor remainders we use, while the global bound on $\sigma'$ enforces that $\sigma$ is Lipschitz and, together with Assumption~\ref{ass:input}a, keeps the forward and backward passes square integrable. These assumptions are relatively mild for smooth activations and are satisfied for several standard activation functions, including $\tanh$, $\operatorname{GeLU}$, and $\operatorname{SiLU/Swish}$.

Unless indicated otherwise, $\|v\|$ denotes the Euclidean norm of a vector, $\|v\|_\infty\doteq\max_i|v_i|$, and $\langle u,v\rangle$ represents the Euclidean inner product. For a matrix, $\|A\|_{\operatorname{op}}$ and $\|A\|_F$ denote the operator and Frobenius norms, and $\langle A,B\rangle_F\doteq\operatorname{tr}\left(A^\top B\right)$ is the Frobenius inner product. Every matrix gradient will be taken with respect to this inner product. For a scalar function $g:\R\to\R$, $\|g\|_\infty\doteq\sup_{u\in\R}|g(u)|$.

For a scalar random variable $Z$, we write $\|Z\|_{L^p}\doteq\left(\E\left[|Z|^p\right]\right)^{1/p}$. For a random vector or matrix $Z$, the same notation denotes the corresponding Bochner norm, using the Euclidean or Frobenius norm pointwise. When the pointwise norm matters, we display it explicitly as a nested norm: for example, $\left\|\,\|Z\|\,\right\|_{L^p}\doteq\left(\E\left[\|Z\|^p\right]\right)^{1/p}$ for a random vector $Z$, while $\left\|\,\|Z\|_\infty\,\right\|_{L^p}\doteq\left(\E\left[\|Z\|_\infty^p\right]\right)^{1/p}$. Expectations and $L^p$ norms are taken with respect to the population law unless another measure is specified. We will use the notation $[m]\doteq\left\{1,\ldots,m\right\}$ and the convention that an empty product is taken to equal one.

We study training by gradient flow on the population squared loss:
\begin{equation*}
	\mathcal L (W) = \frac12 \E_{x,y}\left[(f(x;W) - y)^2\right], \qquad \dot W_\ell = -\nabla_{W_\ell}\mathcal L,
\end{equation*}
and assume that the target is generated by a single hidden layer teacher with $r^*$ modes:
\begin{equation}\label{eq:teacher}
		y = \sum_{k=1}^{r^*}\beta_k \sigma\left(\frac{v_k^\top x}{s_k}\right), \qquad s_k^2 \doteq v_k^\top \Sigma v_k, \quad \beta_k \in \R\setminus\left\{0\right\},
\end{equation}
where $r^*\in\mathbb N$, $v_1, \dots, v_{r^*}$ are (not necessarily linearly independent) unit vectors and $\Sigma$ is the input covariance. Each mode is a nonlinear feature along one input direction $v_k$, with signed amplitude $\beta_k$. We denote by $g_k \doteq v_k^\top x/s_k$ the standardized index of mode $k$, so that $\operatorname{Var}[g_k]=1$. Signed coefficients and repeated directions are permitted. In this paper, we are primarily concerned with the time at which the network first leaves the saddle. We later study the multimode teacher to determine whether distinct target features emerge through separate escapes or depart from the saddle together (Section~\ref{sec:multi-mode}).

We impose the following assumptions on the input distribution. The first two conditions describe how the input covariance interacts with the teacher directions, the third provides the moment control needed for the nonlinear remainders and the final gives a common normalization across teacher modes (a notational convenience that does not affect our results beyond bookkeeping).

\begin{assumption}\label{ass:input}
	The input $x\in\R^d$ satisfies the following conditions:
	\begin{enumerate}[label=\alph*.]
		\item $\E[x] = 0$ and $\Sigma \doteq \E[xx^\top] \succ 0$, normalized so that $\|\Sigma\|_{\operatorname{op}} = 1$;
		\item for each teacher direction $v_k$, $\E[x \mid v_k^\top x] = \frac{v_k^\top x}{s_k^2} \Sigma v_k$;
			\item every linear image of $x$ has a controlled $2q$-th moment: for every real matrix $P$ with $d$ columns,
			\begin{equation*}
				\left(\E\left[\|Px\|^{2q}\right]\right)^{1/(2q)} \leq \kappa\left(\E\left[\|Px\|^2\right]\right)^{1/2},
			\end{equation*}
			where $\kappa$ is independent of $P$;
		\item the standardized indices $g_1, \dots, g_{r^*}$ share a common law, and $h_\sigma \doteq \E[g\sigma(g)] > 0$, where $g$ denotes a random variable with that common law.
	\end{enumerate}
\end{assumption}

The normalization in Assumption~\ref{ass:input}a fixes the scale of the data; nothing we derive depends on the normalization except through constants, which would otherwise carry powers of $\|\Sigma\|_{\operatorname{op}}$. Assumption~\ref{ass:input}b ensures that conditioning on a teacher index produces a mean input along its covariance-transformed direction $\Sigma v_k$. It is the linearity condition introduced by \citet{Li1991} for sliced inverse regression and studied in high dimensions by \citet{Hall1993, Leeb2013} and holds for every elliptical law. Conversely, \citet{Eaton1986} showed that imposing it in every direction characterizes ellipticity; our requirement only along the $r^*$ teacher directions is therefore strictly weaker.

Assumption~\ref{ass:input}c is the hypercontractivity condition used to control the Taylor remainders of $\sigma$ in $L^2$. It holds for Gaussian inputs and, more generally, for input laws with the corresponding uniform sub-Gaussian moment comparison; for elliptical inputs, it follows from the analogous moment comparison for the radial variable. Taking $P=I$ gives the more familiar finite moment condition $\E[\|x\|^{2q}]<\infty$. The common-law requirement in Assumption~\ref{ass:input}d is likewise automatic for elliptical inputs, since every standardized one-dimensional projection then has the same distribution. We require it only to make $h_\sigma$ a single number; without it, the analysis is unchanged after replacing $h_\sigma$ by $h_{\sigma,k}\doteq\E[g_k\sigma(g_k)]$, so that mode $k$ has effective amplitude $h_{\sigma,k}\beta_k$.

Throughout, we are concerned with the small-initialization regime, in which the layer scales at initialization are small (and possibly vary per layer). We will formalize the conditions we require in Section~\ref{sec:critical-depth}, closer to when we need them.

\section{Homogeneity Defect}\label{sec:homogeneity-defect}

The analysis of deep linear networks is tractable in large part because the layer imbalance $\|W_{\ell+1}\|_F^2 - \|W_\ell\|_F^2$ is exactly conserved by gradient flow. To begin our analysis of deep nonlinear networks, we first introduce a helpful identity that generalizes this layer imbalance conservation law. In particular, the following theorem measures the failure of the conservation for a general activation. We emphasize that it requires no distributional structure beyond the existence of the population gradients and no restriction on the loss function aside from differentiability.

\begin{theorem}\label{thm:imbalance-identity}
		Let $\sigma \in C^1(\R)$, let $\mathcal L = \E[\mathcal C(f(x;W),y)]$ for any differentiable per-example loss $\mathcal C$, and consider the gradient flow $\dot W_\ell = -\nabla_{W_\ell} \mathcal L$. Assume that $\mathcal L$ is differentiable and, for every layer $j$, that the per-example gradient is integrable and satisfies $\nabla_{W_j}\mathcal L=\E[\nabla_{W_j}\mathcal C]$. Define the homogeneity defect
	\begin{equation*}
		\varphi_\sigma(z) \doteq z \sigma'(z) - \sigma(z),
	\end{equation*}
	extended entrywise to vectors. Then for each $\ell \in [L-1]$
	\begin{equation*}
		\frac{d}{dt} \left(\|W_{\ell+1}\|_F^2-\|W_\ell\|_F^2\right) = 2\E\left[\left\langle W_{\ell+1}^\top\nabla_{z_{\ell+1}}\mathcal C,\varphi_\sigma(z_\ell)\right\rangle\right].
	\end{equation*}
\end{theorem}
\begin{proof}
The proof compares the infinitesimal change in the loss under rescaling $W_\ell$ and $W_{\ell+1}$; their difference reflects the failure of $\sigma$ to commute with scaling. Write $\delta^{\mathcal C}_{\ell+1} \doteq \nabla_{z_{\ell+1}}\mathcal C$ for the gradient of the per-example loss with respect to the pre-activation at layer $\ell+1$ and let $\odot$ denote the entrywise product. Fix a layer $j$ and consider scaling only that layer $W_j \mapsto (1+s)W_j$. The curve $s \mapsto (1+s)W_j$ has velocity $W_j$, so by the chain rule
\begin{equation*}
	\frac{d}{ds} \mathcal L((1+s)W_j) \Big|_{s=0} = \langle W_j, \nabla_{W_j} \mathcal L\rangle_F \doteq \mathcal R_j.
\end{equation*}
Differentiating $\|W_j\|_F^2 = \langle W_j, W_j\rangle_F$ with respect to time and substituting $\dot W_j = -\nabla_{W_j}\mathcal L$,
\begin{equation}\label{eq:-2R_j}
	\frac{d}{dt} \|W_j\|_F^2 = 2 \langle W_j, \dot W_j\rangle_F = -2 \langle W_j, \nabla_{W_j} \mathcal L\rangle_F = -2 \mathcal R_j.
\end{equation}
Therefore, it suffices to compare $\mathcal R_\ell$ and $\mathcal R_{\ell+1}$. Scaling $W_\ell$ or $W_{\ell+1}$ leaves every layer before $\ell$ unchanged, and all subsequent computation is affected only through $z_{\ell+1}$:
\begin{equation*}
	\mathcal R_j
	=
	\E\left[
		\left\langle
			\delta^{\mathcal C}_{\ell+1},
			\frac{d}{ds}z_{\ell+1}\bigl(W_j\mapsto(1+s)W_j\bigr)\Big|_{s=0}
		\right\rangle
	\right],
	\qquad j\in\{\ell,\ell+1\},
\end{equation*}
and the two differ only in how the scaling enters $z_{\ell+1} = W_{\ell+1}\sigma(z_\ell)$. Scaling $W_{\ell+1}$ multiplies the output of $\sigma$, while scaling $W_\ell$ multiplies its argument (since $z_\ell$ is linear in $W_\ell$):
\begin{equation*}
	\frac{d}{ds} \left[(1+s)W_{\ell+1}\sigma(z_\ell)\right]_{s=0} = W_{\ell+1}\sigma(z_\ell), \qquad \frac{d}{ds} \left[W_{\ell+1}\sigma((1+s)z_\ell)\right]_{s=0} = W_{\ell+1}\left(z_\ell \odot \sigma'(z_\ell)\right),
\end{equation*}
where the second equality follows by the entrywise chain rule. Using the elementary equality $\langle a, Wb\rangle = \langle W^\top a, b \rangle$, one obtains
\begin{equation*}
	\mathcal R_{\ell+1} = \E[\langle W_{\ell+1}^\top \delta^{\mathcal C}_{\ell+1}, \sigma(z_\ell) \rangle], \qquad \mathcal R_\ell = \E[\langle W_{\ell+1}^\top \delta^{\mathcal C}_{\ell+1}, z_\ell \odot \sigma'(z_\ell)\rangle],
\end{equation*}
and the claim follows immediately by subtracting and using \eqref{eq:-2R_j}:
\begin{equation*}
	\frac{d}{dt} \left(\|W_{\ell+1}\|_F^2-\|W_\ell\|_F^2\right) = 2(\mathcal R_\ell - \mathcal R_{\ell+1}) = 2\E[\langle W^\top_{\ell+1} \delta^{\mathcal C}_{\ell+1}, z_\ell \odot \sigma'(z_\ell) - \sigma(z_\ell)\rangle].
\end{equation*}
\end{proof}

For the squared loss studied here, the bounded derivative of $\sigma$, the moment conditions in Assumption~\ref{ass:input}, and boundedness of the weights on each compact portion of the flow provide the required integrability and justify the displayed population-gradient identity by dominated convergence.

Theorem~\ref{thm:imbalance-identity} makes the conserved case transparent: if $\sigma((1+s)z) = (1+s)\sigma(z)$ (as is the case for $\sigma = \operatorname{Id}: x \mapsto x$), $\mathcal R_\ell = \mathcal R_{\ell+1}$ and the imbalance is constant, recovering the familiar conservation law for homogeneous networks \citep{DuHuLee2018}.

The scalar relation controls only the total size of adjacent layers. Our escape argument will also require directional control, in order to propagate the growth of the signal channel across the bottleneck layers. For this we will want the corresponding matrix imbalance. In fact, Theorem~\ref{thm:imbalance-identity} is the trace of a matrix identity. Extend the homogeneity defect to two arguments by
\begin{equation*}
	\Phi_\sigma(u,v) \doteq \sigma'(u)v - \sigma(v), \qquad \Phi_\sigma(u,u) = \varphi_\sigma(u),
\end{equation*}
and write $\Phi_\sigma(z)$ for the matrix with entries $\left[\Phi_\sigma(z)\right]_{ij} \doteq \Phi_\sigma(z_i,z_j)$, so that its diagonal is $\varphi_\sigma(z)$.

\begin{proposition}\label{prop:imbalance-matrix}
	Work in the setting of Theorem~\ref{thm:imbalance-identity}, write $\xi_\ell \doteq W_{\ell+1}^\top\nabla_{z_{\ell+1}}\mathcal C$, and $\operatorname{Sym}A \doteq \tfrac12(A + A^\top)$. Then for each $\ell \in [L-1]$
	\begin{equation}\label{eq:imbalance-matrix}
		\frac{d}{dt}\left(W_{\ell+1}^\top W_{\ell+1} - W_\ell W_\ell^\top\right) = 2\operatorname{Sym}\E\left[\operatorname{diag}(\xi_\ell)\,\Phi_\sigma(z_\ell)\right].
	\end{equation}
	Taking traces recovers Theorem~\ref{thm:imbalance-identity}.
\end{proposition}
\begin{proof}
	Both matrices act on $\R^{n_\ell}$. Write $\delta_j^{\mathcal C}\doteq\nabla_{z_j}\mathcal C$. Substituting $\dot W_{\ell+1} = -\E[\delta^{\mathcal C}_{\ell+1}h_\ell^\top]$ and $\dot W_\ell = -\E[\delta^{\mathcal C}_\ell h_{\ell-1}^\top]$, and using $h_\ell = \sigma(z_\ell)$, $W_\ell h_{\ell-1} = z_\ell$ and $\delta^{\mathcal C}_\ell = \operatorname{diag}(\sigma'(z_\ell))\xi_\ell$,
	\begin{equation*}
		W_{\ell+1}^\top\dot W_{\ell+1} = -\E\left[\xi_\ell\sigma(z_\ell)^\top\right], \qquad
		\dot W_\ell W_\ell^\top = -\E\left[\operatorname{diag}(\sigma'(z_\ell))\,\xi_\ell z_\ell^\top\right],
	\end{equation*}
	whose difference has $(i,j)$ entry $\E[\xi_{\ell,i}(\sigma'(z_{\ell,i})z_{\ell,j} - \sigma(z_{\ell,j}))] = \E[\operatorname{diag}(\xi_\ell)\Phi_\sigma(z_\ell)]_{ij}$. Adding the transposed identity gives \eqref{eq:imbalance-matrix}. For the last claim, the diagonal of $\operatorname{diag}(\xi_\ell)\Phi_\sigma(z_\ell)$ is $\xi_\ell\odot\varphi_\sigma(z_\ell)$, and $\operatorname{tr}W_{\ell+1}^\top W_{\ell+1} = \|W_{\ell+1}\|_F^2$, $\operatorname{tr}W_\ell W_\ell^\top = \|W_\ell\|_F^2$; substitution gives  exactly Theorem~\ref{thm:imbalance-identity}.
\end{proof}

For the activations we study, $\Phi_\sigma(u,v) = a_q\left(qu^{q-1}v - v^q\right) + O\left((|u|+|v|)^{q+1}\right)$, so every entry of $\Phi_\sigma(z)$ is $O(\|z\|_\infty^q)$, exactly as $\varphi_\sigma(z) = (q-1)a_qz^q + O(z^{q+1})$ is on the diagonal. Near the saddle the matrix imbalance is therefore conserved to the same accuracy as the scalar one. In particular, near a sufficiently small neighborhood of the saddle, both the matrix and scalar imbalances are approximately conserved. The main aim of the next section is 
to show that this high-order defect remains negligible over the entire escape trajectory.

\section{Critical Depth}\label{sec:critical-depth}

Recall from Section~\ref{sec:setup} that we use $g_k$ to mean the standardized index of mode $k$, $q$ the order of the first nonlinear term of $\sigma$, and $r^*$ the number of teacher modes. The symbol $r$, introduced below, denotes the critical depth, which we will sometimes call the number of bottleneck layers. For $1\leq\ell<L$, write $\delta_\ell \doteq \nabla_{z_\ell}f$ for the gradient of the network output, and set $\delta_L\doteq\nabla_{z_L}f=1$.

\subsection{Layer Gains}\label{subsec:layer-gains}
Near the saddle point, every activation is evaluated close to the origin, where $\sigma(u) = \alpha u + O(u^q)$, so the network is close to a linear map of its inputs and (to leading order) can correlate only with the linear parts of the target. Projecting each $\sigma(g_k)$ onto $g_k$ and writing $\sigma_\perp (g) \doteq \sigma(g) - h_\sigma g$ for the residual, the teacher splits as $y = w_\star^\top x + y_\perp$ with
\begin{equation*}
	w_\star \doteq h_\sigma \sum_{k=1}^{r^*} \frac{\beta_k}{s_k}v_k, \qquad y_\perp \doteq \sum_{k=1}^{r^*} \beta_k\sigma_\perp(g_k).
\end{equation*}

Thus, while the network remains near the saddle, the teacher only acts on the network through $w_\star$. We write $s^2_\star \doteq w_\star^\top \Sigma w_\star$ and $g_\star \doteq w_\star^\top x/s_\star$, defined in analogy to the $s_k$ and $g_k$ of Section~\ref{sec:setup}, and assume $w_\star \neq 0$. The observable through which we track the escape is the network's overlap with that direction, $\gamma(W) \doteq \E[f(x;W)g_\star]$.

This condition is immediately satisfied for a single mode, since then $w_\star = \beta_1h_\sigma v_1/s_1$, $s_\star = |\beta_1|h_\sigma$ and $g_\star = \operatorname{sign}(\beta_1)g_1$. With several modes, the vector contributions may cancel; this is treated in Section~\ref{subsec:rate-constant}. In the remainder of this section, there is no dependence on $r^*$, and one may take $r^* = 1$ throughout after replacing that mode by the aggregate direction $g_\star$.

The input enters the network only through inner products, so the vector presented to the first layer is the image of $w_\star$ under the data covariance: by Assumption~\ref{ass:input}a,
\begin{equation*}
	\E[g_\star x] = \frac{\Sigma w_\star}{s_\star},
\end{equation*}
and the gradient drives the first layer towards this vector. Propagating forward through the weights, and propagating the output sensitivity backward, we set
\begin{equation}\label{eq:path-tensors}
	A_0 \doteq \frac{\Sigma w_\star}{s_\star}, \quad A_j \doteq W_jA_{j-1}; \qquad B_L \doteq 1, \quad B_\ell \doteq W_{\ell+1}^\top B_{\ell+1}.
\end{equation}
Consider the linearized network
\begin{equation*}
	f_{\operatorname{lin}}(x;W) \doteq \alpha^{L-1}W_L\cdots W_1x,
\end{equation*}
whose population overlap with the standardized linear signal $w_\star^\top x/s_\star$ is
\begin{equation}\label{eq:lambda}
	\lambda(W) \doteq \E\left[f_{\operatorname{lin}}(x;W)\frac{w_\star^\top x}{s_\star}\right] = \alpha^{L-1}A_L = \alpha^{L-1}\left\langle B_j,A_j\right\rangle , \qquad 0 \leq j \leq L.
\end{equation}
Unlike $\gamma$, which uses the nonlinear network output, $\lambda$ is a polynomial in the weights and its gradient is exactly rank one at every layer. It will play 
an
important role in the proof of our main result: we derive the escape dynamics for $\lambda$ and then transfer them to the nonlinear network by proving that $\gamma=\lambda(1+o(1))$.

We first parameterize the strength and direction of this signal channel through each layer. Layer $j$ is governed by its action on $A_{j-1}$. Motivated by this, for $1\leq j\leq L$ we define the forward gains and their unit directions, together with the input direction and input factor,
\begin{equation*}
		\rho_j \doteq \frac{\|A_j\|}{\|A_{j-1}\|}, \qquad a_j \doteq \frac{A_j}{\|A_j\|}, \qquad a_0 \doteq \frac{A_0}{\|A_0\|}, \qquad u \doteq \|A_0\| = \frac{\|\Sigma w_\star\|}{\sqrt{w_\star^\top \Sigma w_\star}},
\end{equation*}
so that $\|A_j\| = u \prod_{i=1}^j\rho_i$. Note that the normalization $\|\Sigma\|_{\operatorname{op}} = 1$ of Assumption~\ref{ass:input}a gives $\|\Sigma w_\star\| \leq \|\Sigma^{1/2}w_\star\|$, hence $u \leq 1$. 

At the input, we separate the presented signal into its direction $\hat a$, its normalized covariance gain $u$, and its overall scale $\Gamma=s_\star u$.
Concretely, write $\Gamma \doteq s_\star u = \|\Sigma w_\star\|$ and $\hat a \doteq a_0 = A_0/u$ for the unit vector along the presented direction.

The nonlinear error estimates require us to compare amplification away from the signal direction with amplification along it. We use $\nu$ for the covariance-weighted average scale relevant to the Taylor remainders, and $\chi$ for worst-direction operator amplification. Let $P_0\doteq I$ and $P_j \doteq \alpha^jW_j\cdots W_1$ for $1\leq j<L$, so that $P_j\hat a = \alpha^jA_j/u$. We will track two ratios to record how far the forward pass spreads away from that direction:
\begin{equation}\label{eq:nu-chi}
	\nu \doteq \max_{\ell\in[L]} \frac{\|P_{\ell-1}\Sigma^{1/2}\|_F}{\|A_{\ell-1}\|}, \qquad \chi \doteq \max_{\ell\in[L]} \frac{\|P_{\ell-1}\|_{\operatorname{op}}}{\|P_{\ell-1}\hat a\|} \geq 1 .
\end{equation}
The first compares the signal a layer receives across all input directions with the part of it aligned with the teacher. The second compares the largest gain of the linearized map with its gain along $\hat a$. Importantly, both are invariant under rescaling of the layers and both are evaluated along the flow. We will denote by $\nu_0$ the value of $\nu$ at $t=0$.


We may now formalize our small-initialization condition. Our results do not require a fixed law for the initial weights; we impose only that a prefix of the layers attenuates the signal by a factor of order $\varepsilon$ and that the remaining layers transmit it without losing scale. We call $W_1,\ldots,W_r$ the bottleneck layers because their order-$\varepsilon$ gains suppress signal transmission and determine the escape delay.

\begin{assumption}\label{ass:init}
	There exist constants $0 < c \leq C$ and $c_\lambda>0$, independent of $\varepsilon$, such that at $t=0$:
	\begin{enumerate}[label=\alph*.]
			\item $\rho_j, \|W_j\|_{\operatorname{op}} \in [c\varepsilon, C\varepsilon]$ for $j \leq r$;
			\item every singular value of $W_j$ lies in $[c, C]$ for $r < j < L$; when $r < L$, the readout satisfies $\|W_L\|_F \in [c,C]$;
		\item $\lambda(0) \geq c_\lambda u\varepsilon^r$.
	\end{enumerate}
\end{assumption}

Note that Assumption~\ref{ass:init} places conditions only on initialization. It is also deterministic. Scaled orthogonal initialization satisfies its scale and spectral conditions exactly, while standard independent sub-Gaussian layers satisfy the gain and operator-norm bounds by standard concentration bounds (see, e.g., \citet[Theorem 4.4.3]{Vershynin2026}). The overlap bound (Assumption~\ref{ass:init}c) is a nondegeneracy condition, not a high-probability claim; after normalization by $\varepsilon^r$, its probability does not change as $\varepsilon\to0$.

Two standard cases help to clarify the scope of Assumption~\ref{ass:init}. For scaled orthogonal initialization, we multiply the first $r$ weight matrices by $\varepsilon$ and leave the remaining matrices at order-one scale. The scale and spectral conditions then hold with constants uniform in width. The same construction works with Xavier or He initialization at any fixed width: after multiplying the first $r$ matrices by $\varepsilon$, an i.i.d. Gaussian weight matrix is nonsingular almost surely, so Assumption~\ref{ass:init} holds for almost every normalized draw (with constants that may depend on the draw and the width).

\subsection{Escape Law}

The signal that reaches the output is the product of layer gains, so it will be helpful for us to track a scalar summary: the product's $r$-th root,
\begin{equation*}
	X \doteq \left(\prod_{i=1}^r \rho_i \right)^{1/r}, \qquad X(0) \asymp \varepsilon,
\end{equation*}
where $a\lesssim b$ means $a\leq Cb$, $a\gtrsim b$ means $b\lesssim a$, and $a\asymp b$ or $a=\Theta(b)$ means that both comparisons hold; similarly, $a=O(b)$ means $|a|\lesssim b$. A subscript $K$ (e.g. $\lesssim_K$) permits the implicit constant to depend on $K$. Unless uniformity is stated explicitly, the constants may depend on the fixed problem parameters but never on $\varepsilon$, and every $o(\cdot)$ statement is interpreted in the limit declared locally.

$X$ is used in our definition of escape, but the product form is inconvenient to differentiate. Instead, we will use
\begin{equation*}
	Y\doteq\rho_1^2
\end{equation*}
because its derivative has a simple form. The quantity $Y$ measures the strength of the first-layer channel, while $\lambda$ measures the network's overlap with the linear target signal. This overlap drives the growth of the channel, $\dot Y\asymp s_\star\lambda$, and the growing channel in turn amplifies the overlap, $\dot\lambda\gtrsim s_\star u^2Y^{r-1}$. Integrating this feedback loop gives the escape-time law, and the relation $X\asymp\sqrt{Y}$ translates the result back to the escape criterion.

Before proceeding to the general nonlinear result, let us examine the linear case to gain some intuition. Consider a scalar linear chain, initialize its first $r$ weights equally at $w(0)=\varepsilon$, and freeze the product $\kappa\doteq\prod_{\ell>r}w_\ell$ of the remaining weights at a positive order-one value. For a unit-variance input with input-output correlation $\Gamma$, balanced gradient flow of the small weights satisfies
\begin{equation}\label{eq:linear-warmup}
	\dot w=\kappa\left(\Gamma-\kappa w^r\right)w^{r-1}=\kappa\Gamma w^{r-1}+O\left(w^{2r-1}\right).
\end{equation}
The leading equation reaches a fixed small threshold $\vartheta$ in time
\begin{equation*}
	\frac{1}{\kappa\Gamma}\begin{cases}
		\vartheta-\varepsilon, & r=1,\\
		\log\left(\vartheta/\varepsilon\right), & r=2,\\
		\left(\varepsilon^{-(r-2)}-\vartheta^{-(r-2)}\right)/(r-2), & r\geq3.
	\end{cases}
\end{equation*}
Under uniform initialization, $r=L$, and this recovers the logarithmic and power-law deep-linear escape scales \citep{saxe2014exactsolutionsnonlineardynamics,JacotEtAl2021Saddle}. The prefix initialization replaces the total depth by the number of small layers in the final result. Examining the linear case, an immediate observation is that the calculation rests on balanced growth of the small layers and negligible motion of the order-one layers. However, smooth nonlinearities do not in general preserve exact balance and introduce rotating directions and Taylor remainders. Thus, we must prove that these remainders are controlled and that the layers below the critical depth remain close to initialization.

Before formally stating our result, we fix the time window on which it holds and define what we mean by escape. We will care about two scales: the largest bottleneck layer and the largest pre-activation anywhere in the network,
\begin{equation}\label{eq:scales}
	M \doteq \max_{\ell\in[r]} \|W_\ell\|_{\operatorname{op}}, \qquad \zeta \doteq \max_{j\in[L-1]}\left\|\,\|z_j\|_\infty\,\right\|_{L^{2q}} .
\end{equation}
The quantity $M$ controls the bottleneck weights, whereas $\zeta$ measures how far the hidden activations have entered the nonlinear regime; our error terms are polynomial in these two scales. For a small $m_0$, define $\tau_{m_0}\doteq\inf\left\{t\geq0:\max\left(M(t),\zeta(t)\right)\geq m_0\right\}$ and call $[0,\tau_{m_0}]$ the small-signal interval. Because this is a first exit time, its definition assumes no prior control of the trajectory. The two scales become linked once the post-bottleneck hidden layers remain well-conditioned and the readout norm remains of order one: Lipschitzness of $\sigma$ and Assumption~\ref{ass:input}c then imply $\zeta\lesssim M\sqrt{\operatorname{tr}\Sigma}$. Lemma~\ref{lem:channel} establishes these properties on the interval used, until then we retain both $M$ and $\zeta$ in the stopping time.

The radius $m_0$ specifies the neighborhood in which the network remains approximately linear, whereas the initialization enters through the ratio $\nu_0$. Defining the local error level
\begin{equation*}
	\operatorname{err}_0 \doteq m_0^{q-1}\nu_0^2 + \frac{m_0^r}{\Gamma},
\end{equation*}
the first term controls the Taylor remainders, including the contribution of the network's own nonlinear output, whereas the second compares the network's order-$m_0^r$ output with the linear signal scale $\Gamma$. We require $\operatorname{err}_0\leq c_0$ for a sufficiently small constant $c_0$. Since $\nu_0\geq\sqrt{\operatorname{tr}\Sigma}/u$, this condition requires $m_0\sqrt{\operatorname{tr}\Sigma}/u\ll1$ when $q=3$. This will be the only dimension dependence in the local estimates.

Choose an escape threshold $0<\vartheta<m_0$ satisfying $\vartheta\sqrt{\operatorname{tr}\Sigma}\leq c_0m_0$ and define
\begin{equation*}
	\tau \doteq \inf\left\{t \geq 0 \mid X(t) \geq \vartheta\right\}.
\end{equation*}
We call $(m_0,\vartheta)$ admissible when these smallness conditions hold. Lemma~\ref{lem:channel} then ensures that the trajectory remains in the small-signal interval until time $\tau$. Moreover, Lemma~\ref{lem:gamma-scale} gives $\gamma\asymp uX^r$, so $X=\vartheta$ corresponds to the nonlinear network reaching an overlap of order $u\vartheta^r$.

Note that the escape law is an asymptotic in $\varepsilon$, with the admissible local window $(m_0,\vartheta)$ held fixed. Shrinking this window is needed only for the sharper $1+o(1)$ descriptions of the channel and does not affect the escape exponent.

We are now ready to formally state our main result. 

\begin{theorem}[Critical-depth escape law]\label{thm:critical-depth-escape-law}
	Let $r\geq 1$. Under Assumptions~\ref{ass:input} and \ref{ass:init}, fix an admissible small-signal radius $m_0$ and escape threshold $\vartheta$ as above. Then the escape occurs before $\tau_{m_0}$, the bottleneck gains satisfy $\rho_\ell\asymp X\asymp\rho_1$ on $[0,\tau]$, and
	\begin{equation}\label{eq:lambda-Y-growth}
		\lambda(t)^2 \geq \lambda(0)^2 + cu^2\left(Y(t)^r-Y(0)^r\right), \qquad Y \doteq \rho_1^2 .
	\end{equation}
	Accordingly, the escape time satisfies
	\begin{equation*}
		\tau = \begin{cases}
			\Theta\left(\Gamma^{-1}\right), & r=1,\\
			\Theta\left(\Gamma^{-1}\log\left(1/\varepsilon\right)\right), & r=2,\\
			\Theta\left(\Gamma^{-1}\varepsilon^{-(r-2)}\right), & r\geq3.
		\end{cases}
	\end{equation*}
	Here $\Gamma$ and the remaining problem parameters are fixed as $\varepsilon\to0$; uniform comparison across problems with different signal strengths requires the common admissibility bounds stated below.
	If $r<L$, no trajectory-wise readout-alignment hypothesis is required, and at escape $\rho_L(\tau)\asymp1$.
\end{theorem}

Remarkably, Theorem~\ref{thm:critical-depth-escape-law} has no dependence, up to a multiplicative prefactor, on the layer width, the total depth beyond the critical prefix $r$, the particular choice of smooth activation, or the input covariance beyond its role in $\Gamma=\|\Sigma w_\star\|$. The scaling of the escape time is therefore determined by a single scalar and a single integer. This is a rather unintuitive statement: width multiplies the number of directions available to the flow, additional depth multiplies the number of compositions, and different activations reshape the loss landscape, so there is no a priori reason to expect any of these to be irrelevant to the escape time scaling. The fact that they are absent indicates that the escape from the saddle point is not a high-dimensional event, but instead a one-dimensional one, with everything else either irrelevant or absorbed into $\Gamma$.

The inequality \eqref{eq:lambda-Y-growth} states that growth of the first bottleneck gain forces growth of the multilinear signal, $\lambda\gtrsim uY^{r/2}$ once $Y$ exceeds its initialization scale, so the channel cannot expand while its overlap with the presented direction remains small. Substituting this relation into $\dot Y\asymp s_\star\lambda$ yields the three escape regimes. What the theorem does not fix is the scale of that drive: it gives the $\varepsilon$-dependence of the escape time for a fixed problem, and its constants may otherwise vary with the target and input distribution, so the factor $\Gamma^{-1}$ becomes a cross-problem prediction only once the remaining problem parameters are controlled uniformly. This is made precise in the following corollary.

\begin{corollary}\label{cor:uniform-signal-rescaling}
	Fix $L$, $r$, the widths, and the activation, and consider a family of problems with common constants in Assumptions~\ref{ass:input} and \ref{ass:init}. Suppose, uniformly over the family, that $\nu_0$, $\|y_\perp\|_{L^{2q}}/s_\star$, and $\operatorname{tr}\Sigma$ are bounded above and that $u$ and $\Gamma$ are bounded below by a positive constant, and choose a common admissible pair $(m_0,\vartheta)$. Then there exist $c,C,\varepsilon_0>0$, common to the family, such that for every $0<\varepsilon<\varepsilon_0$,
	\begin{equation*}
		\frac{c}{\Gamma}\Psi_r(\varepsilon)\leq\tau\leq\frac{C}{\Gamma}\Psi_r(\varepsilon), \qquad \Psi_r(\varepsilon)\doteq\begin{cases}
			1, & r=1,\\
			\log\left(1/\varepsilon\right), & r=2,\\
			\varepsilon^{-(r-2)}, & r\geq3.
		\end{cases}
	\end{equation*}
\end{corollary}
\begin{proof}
The corollary follows quickly after the proof of Theorem~\ref{thm:critical-depth-escape-law} (Section~\ref{sec:pf-critical-depth-escape-law}). In particular, the upper-bound argument gives
	\begin{equation*}
		\dot Y\geq c\Gamma\sqrt{Y^r-Y(0)^r+c'\varepsilon^{2r}},
	\end{equation*}
	with common $c,c'>0$. Integrating from $Y(0)\asymp\varepsilon^2$ to a fixed multiple of $\vartheta^2$ gives $\tau\leq C\Gamma^{-1}\Psi_r(\varepsilon)$. Conversely, the same argument gives $\dot M\leq C\Gamma M^{r-1}$ almost everywhere with a common constant. Integrating from $M(0)\asymp\varepsilon$ to the fixed escape scale gives $\tau\geq c\Gamma^{-1}\Psi_r(\varepsilon)$. The common lower bound on $\Gamma$ and the common admissible pair keep both integrations inside the small-signal interval uniformly over the family.
\end{proof}

Intuitively, $\Gamma=\|\Sigma w_\star\|$ is the strength of the covariance-weighted linear signal presented to the first layer. Strengthening that signal accelerates the same feedback mechanism, rescaling the escape time by $\Gamma^{-1}$ without changing the critical-depth exponent in $\varepsilon$.

The leading escape integral underlying Theorem~\ref{thm:critical-depth-escape-law} gives the finite-threshold prediction
\begin{equation}
		\tau_\vartheta^{\operatorname{pred}}(\varepsilon) = C_r\begin{cases}
			1, & r=1,\\
			\log\left(\vartheta/\varepsilon\right), & r=2,\\
			\varepsilon^{-(r-2)}-\vartheta^{-(r-2)}, & r\geq3.
		\end{cases}
\end{equation}
where $C_r$ is the only undetermined prefactor. Each calibration fixes only a multiplicative prefactor that the asymptotic theorem leaves unspecified; the dependence on initialization scale is evaluated on held-out conditions. In Figure~\ref{fig:critical-depth}, we determine $C_r$ at a single $\varepsilon$ per $r$ and predict the escape times at the other initialization scales, finding close agreement between Theorem~\ref{thm:critical-depth-escape-law} and population gradient flow on finite-width networks.


\begin{figure}[!t]
	\centering
	\includegraphics[width=\linewidth]{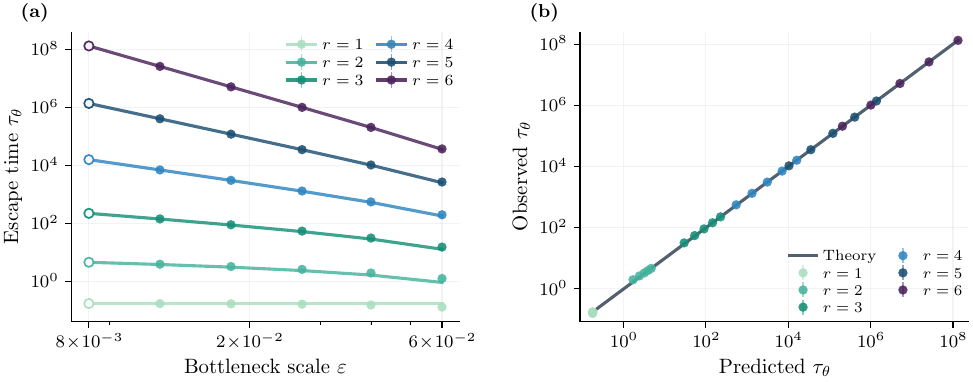}
	\caption{\textbf{Critical depth predicts escape time.} \textbf{(a)} Population-flow escape times (dots) and Theorem~\ref{thm:critical-depth-escape-law} predictions (lines) for $r=1,\dots,6$. For each $r$, the prefactor $C_r$ is fixed at the smallest initialization scale (open dot); the remaining scales are held out. \textbf{(b)} Predicted versus observed escape times on the held-out scales. Dots report medians, and error bars span $12$ independent initializations.}
	\label{fig:critical-depth}
\end{figure}

\FloatBarrier

\subsection{Signal Scale}\label{subsec:rate-constant}

Before the proof of Theorem~\ref{thm:critical-depth-escape-law}, we comment on a few observations about its implications, uses, and some notable properties of $\Gamma$. Theorem~\ref{thm:critical-depth-escape-law} separates the initialization exponent from the signal scale $\Gamma = \|\Sigma w_\star\|$, and Corollary~\ref{cor:uniform-signal-rescaling} makes the factor $\Gamma^{-1}$ uniform over bounded admissible families. The result does not identify the exact multiplicative prefactor, which may depend on the architecture, activation, non-bottleneck gains, and other admissibility parameters. Thus $r$ determines the power of $\varepsilon$, and $\Gamma$ determines the signal rescaling once those parameters are controlled. Across the experiments, our calibration point fixes only the multiplicative prefactor; the dependence on initialization scale and signal geometry is evaluated on held-out conditions.

In this subsection, every unqualified mode index ranges over $[r^*]$, and we write $\beta$ to denote $\left(\beta_1,\ldots,\beta_{r^*}\right)^\top$. Substituting $w_\star = h_\sigma \sum_{k=1}^{r^*} (\beta_k/s_k)v_k$ and setting
\begin{equation*}
	u_k \doteq \frac{\|\Sigma v_k\|}{s_k}, \qquad \hat a_k \doteq \frac{\Sigma v_k}{\|\Sigma v_k\|},
\end{equation*}
gives
\begin{equation}\label{eq:Gamma}
	\Gamma = h_\sigma \left\|\sum_{k=1}^{r^*} \beta_k u_k \hat a_k\right\|.
\end{equation}

Each mode contributes the vector $h_\sigma\beta_ku_k\hat a_k$: its length is $h_\sigma|\beta_k|u_k$, and the sign of $\beta_k$ chooses between the covariance image $\hat a_k$ and its negative. The factor $u_k$ is familiar: it is the input factor $u$ of Section~\ref{subsec:layer-gains} for a teacher carrying the single mode $v_k$. Four features of \eqref{eq:Gamma} are worth commenting on.

First, the magnitude of each vector contribution is the product $h_\sigma|\beta_k|u_k$. The signed orientations of the vectors still matter through their cancellation in \eqref{eq:Gamma}. This also makes clear our earlier normalization $s_k^2 = v_k^\top \Sigma v_k$ in \eqref{eq:teacher}, since without it, $\beta_k$ and $u_k$ are not separately informative.

Second, the geometric factor is not the variance along $v_k$. One might expect the data to present a direction in proportion to $s_k$, but for a unit vector $v_k$ Cauchy--Schwarz gives $s_k^2 = v_k^\top \Sigma v_k \leq \|\Sigma v_k\|$, and therefore
\begin{equation*}
	u_k = \frac{\|\Sigma v_k\|}{s_k} \geq s_k,
\end{equation*}
with equality when $v_k$ is an eigenvector of $\Sigma$. A misaligned teacher direction is thus always presented more strongly than its own variance would suggest.

Let $\left\{(\lambda_i,e_i)\right\}_{i=1}^d$ be an orthonormal eigensystem of $\Sigma$, and write $\mu_k$ for the measure placing mass $\langle v_k, e_i\rangle^2$ at $\lambda_i$. Then $s_k^2 = \int \lambda \,d\mu_k$ while $u_k^2 = \int\lambda^2d\mu_k/\int\lambda\,d\mu_k$, and hence
\begin{equation*}
	u_k^2-s_k^2=\frac{\operatorname{Var}_{\mu_k}(\lambda)}{s_k^2}.
\end{equation*}
The gap therefore vanishes exactly when $v_k$ is an eigenvector of $\Sigma$ (allowing arbitrary directions within a degenerate eigenspace).


Third, the sum in \eqref{eq:Gamma} is a vector sum. When the $v_k$ are orthonormal eigenvectors of $\Sigma$ with eigenvalues $\lambda_k$, the images $\hat a_k$ are orthonormal as well, and \eqref{eq:Gamma} simplifies to $\Gamma = h_\sigma(\sum_k\beta_k^2\lambda_k)^{1/2}$. In particular, when $\Sigma = I$, the mode amplitudes enter the escape law only through their Euclidean norm, since $\Gamma=h_\sigma\|\beta\|_2$. In general, the cross-terms between modes need not vanish. A pair of modes reduces $\Gamma^2$ exactly when $\beta_j\beta_k\langle\hat a_j,\hat a_k\rangle<0$, and reinforces it when this quantity is positive.

Consequently,
\begin{equation*}
	\Gamma^2 = h_\sigma^2\sum_{j,k}\beta_j\beta_k u_ju_k \left\langle \hat a_j, \hat a_k\right\rangle
\end{equation*}
may be arbitrarily small relative to $\sum_k (\beta_ku_k)^2$. The extreme case $w_\star = 0$ is the degeneracy we set aside in Section~\ref{subsec:layer-gains}: the linear part of the teacher vanishes, the drive in Lemma~\ref{lem:gammadot-scale} is empty, and the leading-order analysis is uninformative. Escape (should it occur) is then driven at higher order and on a different timescale.

Fourth, $\Gamma$ may be computed from the input-output law without resolving the teacher into modes.

\begin{corollary}\label{cor:general-targets}
	Let $x$ satisfy the centering and covariance requirements of Assumption~\ref{ass:input}a and the moment bound in Assumption~\ref{ass:input}c. Let $y\in L^{2q}$, and define
	\begin{equation*}
		w_\star\doteq\Sigma^{-1}\E[yx]\neq0, \qquad s_\star^2\doteq w_\star^\top\Sigma w_\star, \qquad y_\perp\doteq y-w_\star^\top x.
	\end{equation*}
	Suppose $\|y_\perp\|_{L^{2q}}\leq C_ys_\star$ for a fixed constant $C_y$, and suppose Assumption~\ref{ass:init} holds with this $w_\star$. Then the conclusions of Theorem~\ref{thm:critical-depth-escape-law} hold with
	\begin{equation}\label{eq:Gamma-corr}
		\Gamma=\left\|\E[yx]\right\|.
	\end{equation}
	The same conclusion holds when $\E$ is the empirical average over a centered finite sample whose empirical covariance is positive definite on the working feature space (provided the same residual and initialization conditions hold after covariance normalization).
\end{corollary}

\begin{proof}
	The definition of $w_\star$ gives $\E[y_\perp x]=\E[yx]-\Sigma w_\star=0$. The proof of Theorem~\ref{thm:critical-depth-escape-law} uses the structured teacher only through this orthogonality and the stated $L^{2q}$ bound on $y_\perp$; the mode-specific conditions in Assumption~\ref{ass:input}b,d are otherwise unnecessary. Since $\Sigma w_\star=\E[yx]$, its signal scale is \eqref{eq:Gamma-corr}. For a centered finite sample, empirical averaging defines a finite-support population law, and equivalence of moments on that finite support supplies Assumption~\ref{ass:input}c with a finite constant. Positive definiteness on the working feature space supplies Assumption~\ref{ass:input}a after normalization, so the same proof applies.
\end{proof}

The structured teacher \eqref{eq:teacher} satisfies the residual condition for every fixed nondegenerate choice of modes, with a constant that may depend on their amplitudes and geometry. Assumption~\ref{ass:input}b makes its decomposition compatible with the definition of $w_\star$, so that \eqref{eq:Gamma} resolves \eqref{eq:Gamma-corr} into per-mode contributions. Equation~\eqref{eq:Gamma-corr} also lends itself naturally to comparison with matrix-output deep linear networks, in which the same correlation is a matrix rather than a vector; we discuss this in more detail in Section~\ref{sec:multi-mode}.

Importantly, the signal scale must be vector-valued, not a sum of independent mode contributions. In Figure~\ref{fig:signal-geometry}a, rotating two equal-amplitude teacher modes relative to one another leaves each mode's individual contribution unchanged but changes the angle between their covariance images. Any summary that scored the modes independently would therefore stay fixed, but $\Gamma$ falls as the two images increasingly cancel, and the observed population-flow escape time grows by the matching factor. Figure~\ref{fig:signal-geometry}b tests this prediction more broadly, crossing several input covariances with several mode orientations; each condition gives one point, with the predicted $C/\Gamma$ on the horizontal axis and the observed escape time on the vertical axis, $C$ fixed once from a separate isotropic single-mode condition and the diagonal marking exact agreement. Both the anisotropy and the mode-interference conditions sit close to that diagonal, so the same inverse-$\Gamma$ relation holds across covariance and mode configurations. In addition, Appendix~\ref{app:signal-ablation} compares $\Gamma$ with alternative scalar summaries and shows that only the vector norm retains both the strength and cancellation of the covariance-transformed modes. 

\begin{figure}[t]
	\centering
	\includegraphics[width=\linewidth]{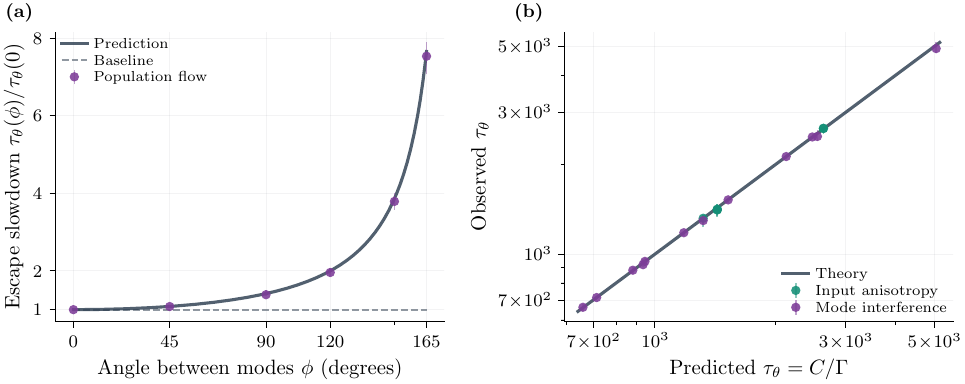}
	\caption{\textbf{$\Gamma$ predicts escape time.} \textbf{(a)} Population-flow escape slowdowns (dots) and the vector prediction (solid line) as two equally weighted teacher modes rotate toward cancellation; the dashed baseline ignores the angle between modes. \textbf{(b)} Predicted versus observed escape times under input anisotropy and mode interference. The prefactor in $\tau_\vartheta=C/\Gamma$ is fixed by a separate isotropic single-mode condition; all plotted conditions are held out. Dots report medians, and error bars span eight independent initializations.}
	\label{fig:signal-geometry}
\end{figure}

Importantly, this same information quantifies what notions of escape Theorem~\ref{thm:critical-depth-escape-law} applies to. Near the saddle the network does not explore the ambient space: its linear functional $p \doteq \alpha^{L-1}(W_L\cdots W_1)^\top$ grows essentially along the single direction $\hat a = \Sigma w_\star/\|\Sigma w_\star\|$ (we show this in Proposition~\ref{prop:alignment}), so any criterion requiring more than growth along $\hat a$ falls outside the regime controlled by our theory. Write $\mathcal L_{\operatorname{lin}}(\rho) \doteq \tfrac12(\rho\hat a - w_\star)^\top\Sigma(\rho\hat a - w_\star)$ for the linear part of the loss at amplitude $\rho$ along that direction. Minimizing over $\rho$ and normalizing by the value at $\rho = 0$ gives
\begin{equation}\label{eq:Xi}
	\frac{\min_\rho\mathcal L_{\operatorname{lin}}(\rho)}{\mathcal L_{\operatorname{lin}}(0)} = 1 - \Xi, \qquad \Xi \doteq \frac{\left(w_\star^\top\Sigma^2w_\star\right)^2}{\left(w_\star^\top\Sigma w_\star\right)\left(w_\star^\top\Sigma^3w_\star\right)}.
\end{equation}

Using the eigensystem introduced earlier, write $\mu_\star$ for the measure placing mass $\langle w_\star,e_i\rangle^2$ at $\lambda_i$, so that
\begin{equation*}
	w_\star^\top\Sigma^jw_\star=\int\lambda^j\,d\mu_\star.
\end{equation*}
It is then clear that the quantity $\Xi$ is the ratio of the two sides of the Cauchy--Schwarz inequality applied to $\mu_\star$ under the splitting $\lambda^2=\lambda^{1/2}\cdot\lambda^{3/2}$,
\begin{equation*}
	\left(\int\lambda^2d\mu_\star\right)^2 \leq \left(\int\lambda\,d\mu_\star\right)\left(\int\lambda^3d\mu_\star\right),
\end{equation*}
so $\Xi\leq1$, with equality exactly when this bound is tight, which happens iff $\mu_\star$ is a point mass, that is, iff $w_\star$ is an eigenvector of $\Sigma$. Thus $\Xi$ measures how close the teacher direction is to an eigendirection of the covariance: it equals $1$ at alignment and decreases as $w_\star$ spreads across eigenvalues, with $\Xi\equiv1$ when $\Sigma=I$ as expected.

Before returning to the nonlinear network, we characterize exactly how much the loss can decrease along this leading linear direction.

\begin{proposition}\label{prop:loss-ray}
	Define
	\begin{equation*}
		\mathcal L_0\doteq\mathcal L(W=0)=\frac12\left(s_\star^2+\E[y_\perp^2]\right), \qquad b_{\operatorname{ray}}\doteq\hat a^\top\Sigma w_\star=\Gamma, \qquad c_{\operatorname{ray}}\doteq\hat a^\top\Sigma\hat a,
	\end{equation*}
	where $W=0$ means that every weight matrix is zero. For the linear predictor $x\mapsto\rho\hat a^\top x$,
	\begin{equation}\label{eq:ray-loss}
		\mathcal L_{\operatorname{ray}}(\rho)=\mathcal L_0-\rho b_{\operatorname{ray}}+\frac12c_{\operatorname{ray}}\rho^2.
	\end{equation}
	Its maximum fractional decrease is
	\begin{equation}\label{eq:eta-max}
		\eta_{\max}\doteq\frac{b_{\operatorname{ray}}^2}{2c_{\operatorname{ray}}\mathcal L_0}=\Xi\,\frac{s_\star^2}{s_\star^2+\E[y_\perp^2]}\leq\Xi.
	\end{equation}
	For every $0<\eta<\eta_{\max}$, the first amplitude on the ray attaining $\mathcal L_{\operatorname{ray}}(\rho)=(1-\eta)\mathcal L_0$ is
	\begin{equation}\label{eq:rho-eta}
		\rho_\eta=\frac{b_{\operatorname{ray}}-\sqrt{b_{\operatorname{ray}}^2-2c_{\operatorname{ray}}\eta\mathcal L_0}}{c_{\operatorname{ray}}}.
	\end{equation}
\end{proposition}
\begin{proof}
	The decomposition $y=w_\star^\top x+y_\perp$ is orthogonal in $L^2$ because $\E[y_\perp x]=0$. Expanding the square gives \eqref{eq:ray-loss}. The concave reduction $\rho b_{\operatorname{ray}}-c_{\operatorname{ray}}\rho^2/2$ is maximized at $\rho=b_{\operatorname{ray}}/c_{\operatorname{ray}}$ with value $b_{\operatorname{ray}}^2/(2c_{\operatorname{ray}})$; substituting $\hat a=\Sigma w_\star/\|\Sigma w_\star\|$ gives \eqref{eq:eta-max} by substitution into \eqref{eq:Xi}. Solving the resulting quadratic and taking its smaller root gives \eqref{eq:rho-eta}.
\end{proof}

Proposition~\ref{prop:loss-ray} tracks how much loss the leading linear direction alone can remove: a fractional decrease $\eta$ is attainable on this ray only when $\eta<\eta_{\max}$, although that bound itself is uninformative about whether the nonlinear trajectory actually reaches the corresponding amplitude $\rho_\eta$ while still inside the small-signal interval. Invoking Theorem~\ref{thm:critical-depth-escape-law} at that endpoint requires two further conditions: $\rho_\eta$ must be reached before $\tau_{m_0}$, and the alignment and nonlinear approximation errors must stay below the target loss decrease. Outside those conditions, the loss can still fall (through rotation of $p$ or through nonlinear correlation with $y_\perp$); however, neither is something the first-escape theorem controls.

\begin{figure}[t]
	\centering
	\includegraphics[width=\linewidth]{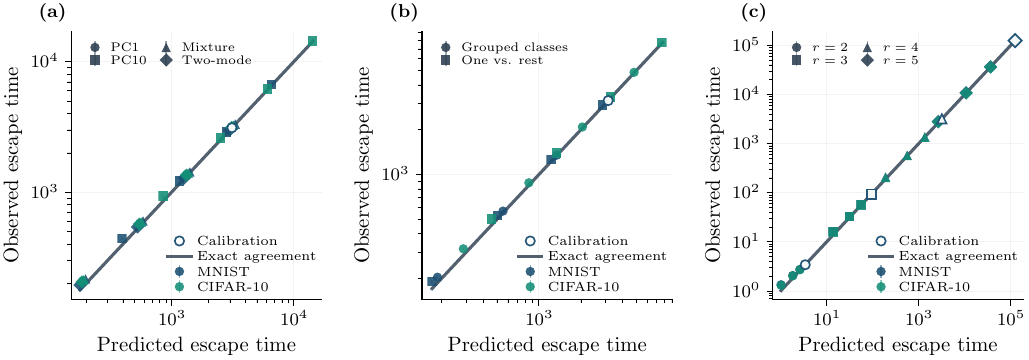}
	\caption{\textbf{$\widehat\Gamma$ and $r$ predict escape on real covariates.} \textbf{(a)} Predicted versus observed escape times on MNIST and CIFAR-10 across four controlled targets and multiple initialization scales. One open MNIST point at the smallest initialization scale fixes the prefactor; all other conditions are held out. \textbf{(b)} The same comparison for targets constructed from the class labels, with a separate open MNIST calibration point. \textbf{(c)} Critical-depth predictions for $r=2,\ldots,5$. One open MNIST point per $r$, at the smallest initialization scale, fixes the $r$-dependent prefactor; the remaining scales and CIFAR-10 conditions are held out. Dots report medians, and error bars span the eight independent initializations.}
	\label{fig:real-covariates}
\end{figure}

Because \eqref{eq:Gamma-corr} needs no decomposition of the teacher into modes, its signal scale can be read off directly from a finite training set as $\widehat\Gamma=\left\|n^{-1}\sum_i y_ix_i\right\|$. Figure~\ref{fig:real-covariates} evaluates this empirical quantity on fixed MNIST and CIFAR-10 covariates. Figure~\ref{fig:real-covariates}a varies controlled target directions on MNIST and CIFAR-10 and finds escape time scaling as $\widehat\Gamma^{-1}$ across datasets and initialization scales. Panel (b) finds the same agreement when the targets are centered class labels rather than predefined functions of the covariates. Panel (c) instead varies the critical depth and recovers the predicted dependence on $\varepsilon$, with one prefactor per $r$ as Theorem~\ref{thm:critical-depth-escape-law} allows. Appendix~\ref{app:real-robustness} shows that the predictions persist under discrete gradient descent and stochastic gradient descent, and under changes to the threshold $\vartheta$.

\subsection{Proof of Theorem~\ref{thm:critical-depth-escape-law}}\label{sec:pf-critical-depth-escape-law}

We now prove Theorem~\ref{thm:critical-depth-escape-law}. The argument reduces a high-dimensional gradient flow to a pair of coupled scalar inequalities via the observation that near the saddle, every layer gradient collapses to rank one, tracking the same direction as the multilinear overlap $\lambda$. With that fact in hand, the bottleneck gains, the post-bottleneck layers, and the nonlinear overlap $\gamma$ can all be computed from $\lambda$, and the escape law follows by integration.

In particular,
define the exact rank-one gradients of $\lambda$ and their squared scale by
\begin{equation}\label{eq:H-S}
	H_\ell \doteq \nabla_{W_\ell}\lambda = \alpha^{L-1}B_\ell A_{\ell-1}^\top, \qquad \mathcal S(W) \doteq \sum_{\ell=1}^L\|H_\ell\|_F^2.
\end{equation}
The constants that follow are fixed in a specific order, assuming each is available before the next is chosen. First $K_0>1$ is fixed, large enough that Assumption~\ref{ass:init} makes the bounds at $t=0$ strict. Then $K_\star$, depending only on the fixed problem parameters, is determined from the sharpened estimates of Lemma~\ref{lem:channel}. With both in hand we fix
\begin{equation*}
	K>2\max\left(K_0,K_\star\right),
\end{equation*}
and only then choose the local window $(m_0,\vartheta)$, small enough that the $K$-dependent error level this produces is negligible. Because $K$ is already fixed by that point, shrinking the window can always absorb whatever $K$-dependent constants the argument needs along the way. Write $R\doteq\varepsilon+\rho_1$. Let $\mathcal B_K(t)$ denote the condition that, for every time $s\leq t$, all of the following bounds hold:
\begin{equation}\label{eq:provisional-bounds}
	\frac{R}{K}\leq\rho_\ell\leq\|W_\ell\|_{\operatorname{op}}\leq KR \quad (\ell\leq r), \qquad \frac{uR^r}{K}\leq\lambda\leq KuR^r, \qquad \nu\leq K\nu_0, \qquad \chi\leq K,
\end{equation}
	together with $\|W_\ell(s)-W_\ell(0)\|_F\leq m_0$ for $\ell>r$. Under $\mathcal B_K(t)$, Weyl's inequality keeps every post-bottleneck hidden matrix $W_\ell$, $r<\ell<L$, uniformly well-conditioned, while the readout norm remains bounded away from both zero and infinity. Consequently,
\begin{equation}\label{eq:post-bottleneck-B}
	c_K\leq\|B_r\|\leq C_K.
\end{equation}
The backward tensor $B_\ell$ propagates the output sensitivity from the readout to layer $\ell$. For $\ell\leq r$, this path crosses the remaining $r-\ell$ bottleneck layers, so its natural scale is $R^{r-\ell}$. The operator-norm bounds in $\mathcal B_K$ give the upper estimate. For the lower estimate, the relation $\lambda=\alpha^{L-1}\langle B_\ell,A_\ell\rangle$ and Cauchy--Schwarz give $\|B_\ell\|\|A_\ell\|\geq\alpha^{-(L-1)}\lambda$. Combining this inequality with the lower bound on $\lambda$ and the upper bounds on the forward gains yields
\begin{equation}\label{eq:coarse-B}
	c_KR^{r-\ell}\leq\|B_\ell\|\leq C_KR^{r-\ell}.
\end{equation}
For $r<\ell\leq L$, the post-bottleneck matrices remain well-conditioned and the readout has order-one norm, so $\|B_\ell\|\asymp1$. Thus the backward signal remains controlled without assuming that the forward readout direction stays aligned along the trajectory.

The conditional estimates below have error level
\begin{equation}\label{eq:eta-K}
	\eta_K \doteq C_K\left(m_0^{q-1}\nu_0^2+\frac{m_0^r}{\Gamma}\right)=C_K\operatorname{err}_0.
\end{equation}
After fixing the problem and the coarse constant $K$, we choose the admissible local window $(m_0,\vartheta)$ so that $\eta_K$ is sufficiently small, and then take $\varepsilon$ small relative to that window. All subsequent $O(\eta_K)$ errors are uniform in $\varepsilon$.

\begin{lemma}\label{lem:grad-scale}
	Suppose $\mathcal B_K(t)$ holds and $t\leq\tau_{m_0}$. Then, for every $\ell$ and every $s\leq t$,
	\begin{equation}\label{eq:flow-H}
		\nabla_{W_\ell}\mathcal L=-s_\star H_\ell+E_\ell, \qquad \|E_\ell\|_F\leq\eta_Ks_\star\|H_\ell\|_F,
	\end{equation}
	and therefore
	\begin{equation}\label{eq:lambda-flow}
		(1-\eta_K)s_\star\mathcal S\leq\dot\lambda\leq(1+\eta_K)s_\star\mathcal S.
	\end{equation}
\end{lemma}
\begin{proof}
	The error may come from three sources: the nonlinear remainder in the forward pass, the nonlinear remainder in the backpropagated derivatives, and the self-interaction generated by the network's current output. We bound them in that order, first obtaining a generic approximation for the pathwise layer gradient and then applying it to the target and self-interaction terms in the loss.

	Write $\varrho(z)\doteq\sigma(z)-\alpha z$, $e_j\doteq h_j-P_jx$ for $1\leq j<L$, and $\tilde e_\ell\doteq\delta_\ell-\alpha^{L-\ell}B_\ell$ for $1\leq\ell\leq L$. The forward recursion is $e_j=\varrho(z_j)+\alpha W_je_{j-1}$ with $e_0=0$, and therefore
	\begin{equation*}
		\|e_j\|\leq\sum_{i\leq j}\alpha^{j-i}\left(\prod_{i<k\leq j}\|W_k\|_{\operatorname{op}}\right)\|\varrho(z_i)\|.
	\end{equation*}
	Bounded derivatives give $|\varrho(v)|\leq C|v|^q$, so H\"older's inequality yields
	\begin{equation*}
		\|\varrho(z_i)\|_{L^2}\leq C\zeta^{q-1}\left\|\,\|z_i\|\,\right\|_{L^{2q}}.
	\end{equation*}
	Since $\sigma$ is Lipschitz and Assumption~\ref{ass:input}c applies to $x$, the latter norm satisfies
	\begin{equation*}
		\left\|\,\|z_i\|\,\right\|_{L^{2q}}\leq C\left(\prod_{k\leq i}\|W_k\|_{\operatorname{op}}\right)\sqrt{\operatorname{tr}\Sigma}.
	\end{equation*}
	Under $\mathcal B_K(t)$, the amplification from layer $i$ to layer $j$ in an arbitrary direction is at most a constant factor larger than its amplification along the presented direction. Thus each term in the preceding sum is controlled by $\|A_j\|$, up to a factor depending on $K$. Moreover, the $\ell=1$ term in the definition of $\nu$ gives $\nu\geq\sqrt{\operatorname{tr}\Sigma}/u$. Therefore
	\begin{equation}\label{eq:forward-backward-errors}
		\|e_j\|_{L^2}\lesssim_K\zeta^{q-1}\nu\|A_j\|.
	\end{equation}

	For the backward recursion, set $D_i\doteq\operatorname{diag}(\sigma'(z_i))$ for $i\in[L-1]$. For $\ell<L$, expand $\delta_\ell=D_\ell W_{\ell+1}^\top D_{\ell+1}\cdots D_{L-1}W_L^\top$ over the first position at which $D_i-\alpha I$ is inserted; at the output, $\tilde e_L=\delta_L-B_L=0$ exactly. Using
	\begin{equation*}
		\left\|D_i-\alpha I\right\|_{\operatorname{op}}\leq C\|z_i\|_\infty^{q-1},
	\end{equation*}
	together with the coarse lower bounds on $\|B_\ell\|$, gives both a deterministic bound $\|\tilde e_\ell\|\lesssim_K\|B_\ell\|$ and the sharper estimate
	\begin{equation}\label{eq:backward-error}
		\left\|\,\|\tilde e_\ell\|\,\right\|_{L^2}\lesssim_K\zeta^{q-1}\|B_\ell\|.
	\end{equation}

	Define
	\begin{equation*}
		\Delta_\ell\doteq\delta_\ell h_{\ell-1}^\top-\alpha^{L-\ell}B_\ell(P_{\ell-1}x)^\top.
	\end{equation*}
	The exact decomposition
	\begin{equation*}
		\Delta_\ell=\alpha^{L-\ell}B_\ell e_{\ell-1}^\top+\tilde e_\ell(P_{\ell-1}x)^\top+\tilde e_\ell e_{\ell-1}^\top
	\end{equation*}
	isolates one small factor in every term. Cauchy--Schwarz controls the first and third terms, using the deterministic bound on $\tilde e_\ell$ in the third. For the second, H\"older's inequality with exponents $2q$, $2$ and $2q/(q-1)$ places a scalar multiplier $w$ in $L^{2q}$, $\tilde e_\ell$ in $L^2$ and $P_{\ell-1}x$ in $L^{2q/(q-1)}$; Assumption~\ref{ass:input}c then returns the last factor to $L^2$. It follows, then, that whenever $\|w\|_{L^{2q}}\leq1$,
	\begin{equation}\label{eq:Delta-bound}
		\left\|\E\left[w\left(\delta_\ell h_{\ell-1}^\top-\alpha^{L-\ell}B_\ell(P_{\ell-1}x)^\top\right)\right]\right\|_F\lesssim_K\zeta^{q-1}\nu\|H_\ell\|_F.
	\end{equation}

	Now split $y=s_\star g_\star+y_\perp$. The linear term of the target contributes exactly $-s_\star H_\ell$. Since $\E[y_\perp x]=0$, the linear part of the residual contribution vanishes and \eqref{eq:Delta-bound} bounds the remainder by $C_Ks_\star\zeta^{q-1}\nu\|H_\ell\|_F$ (here the fixed constant absorbs $\|y_\perp\|_{L^{2q}}/s_\star$).

	It remains only to bound the network self-interaction. Write $f=p^\top x+W_Le_{L-1}$ with $p\doteq\alpha^{L-1}(W_L\cdots W_1)^\top$. The exact linear term is $\alpha^{L-\ell}B_\ell(P_{\ell-1}\Sigma p)^\top$. Since $\|p\|\lesssim_KM^r$ and
	\begin{equation*}
		\|P_{\ell-1}\|_{\operatorname{op}}\leq\chi\|P_{\ell-1}\hat a\|=\chi\alpha^{\ell-1}\frac{\|A_{\ell-1}\|}{u},
	\end{equation*}
	its norm is at most $C_K\chi M^r\|H_\ell\|_F/u$. Applying \eqref{eq:Delta-bound} to $p^\top x$ gives a smaller term of the same order. Finally, \eqref{eq:forward-backward-errors}, the deterministic backward bound, and Cauchy--Schwarz give
	\begin{equation*}
		\left\|\E\left[\left(W_Le_{L-1}\right)\delta_\ell h_{\ell-1}^\top\right]\right\|_F\lesssim_K\zeta^{q-1}\nu^2uR^r\|H_\ell\|_F.
	\end{equation*}
	Recalling $\Gamma=s_\star u$, using $R\leq m_0$, $\nu\geq1$ and $\chi\leq K$, and absorbing fixed constants into $C_K$ gives \eqref{eq:flow-H}. Finally,
	\begin{equation*}
		\dot\lambda=-\sum_\ell\left\langle H_\ell,\nabla_{W_\ell}\mathcal L\right\rangle_F=s_\star\mathcal S-\sum_\ell\langle H_\ell,E_\ell\rangle_F,
	\end{equation*}
	and Cauchy--Schwarz gives \eqref{eq:lambda-flow}.
\end{proof}

The preceding gradient expansion assumes that the provisional bounds in $\mathcal B_K(t)$ hold. The next lemma removes this condition by showing that none of these bounds can fail before the trajectory leaves the small-signal region.

\begin{lemma}\label{lem:channel}
	For $m_0$ sufficiently small, $\mathcal B_K(t)$ cannot fail before $\tau_{m_0}$. Throughout the small-signal interval, $\lambda$ and $Y=\rho_1^2$ are nondecreasing and
	\begin{equation}\label{eq:Y-lambda-flow}
		\dot Y=2s_\star\lambda\left(1+O(\eta_K)\right), \qquad \dot\lambda\gtrsim s_\star u^2R^{2r-2}.
	\end{equation}
	Moreover, for every $\ell\leq r$,
	\begin{equation}\label{eq:channel}
		\|W_\ell\|_{\operatorname{op}}\leq\rho_1\left(1+O(\eta_K)\right)+C\varepsilon, \qquad \rho_\ell^2\geq\rho_1^2\left(1-O(\eta_K)\right)-C\varepsilon^2, \qquad \rho_\ell\gtrsim\frac{\varepsilon^r}{R^{r-1}},
	\end{equation}
	so that $\rho_\ell\asymp\|W_\ell\|_{\operatorname{op}}\asymp R$, $\chi\lesssim1$ and $\nu\lesssim\nu_0$. Every layer after the critical depth satisfies
	\begin{equation}\label{eq:post-bottleneck-displacement}
		\|W_\ell(t)-W_\ell(0)\|_F\lesssim R(t)^2, \qquad \ell>r.
	\end{equation}
\end{lemma}

The proof proceeds by contradiction from the first time that a bound in $\mathcal B_K(t)$ becomes tight. Up to that time, Lemma~\ref{lem:grad-scale} controls the gradient closely enough to make $\lambda$ and $\rho_1$ grow together. Their coupled growth bounds the accumulated motion of the weights; Proposition~\ref{prop:imbalance-matrix} then transfers the first-layer estimate across the bottleneck layers, while the layers beyond the bottleneck move only at the smaller scale $R^2$. These estimates strengthen every bound defining $\mathcal B_K(t)$; a contradiction.

\begin{proof}[Proof of Lemma~\ref{lem:channel}]
	Let $\tau_K$ be the first time at which $\mathcal B_K$ fails, and work on $[0,\tau_K]$. Since $A_0$ is fixed, differentiating $Y=\|W_1A_0\|^2/u^2$ and inserting \eqref{eq:flow-H} gives
	\begin{equation*}
		\dot Y=2s_\star\alpha^{L-1}\langle A_1,B_1\rangle+O\left(\eta_Ks_\star\rho_1\|H_1\|_F\right)=2s_\star\lambda\left(1+O(\eta_K)\right).
	\end{equation*}
	The last comparison uses $\rho_1\|H_1\|_F\lesssim_KuR^r\lesssim_K\lambda$ from \eqref{eq:provisional-bounds}. Thus $Y$ is nondecreasing when $\eta_K$ is small. On the other hand, \eqref{eq:post-bottleneck-B} and the lower gain bounds in $\mathcal B_K$ give
	\begin{equation*}
		\|H_r\|_F=\alpha^{L-1}\|B_r\|\,\|A_{r-1}\|\gtrsim_KuR^{r-1}.
	\end{equation*}
	Using only this one term in \eqref{eq:lambda-flow} proves the second inequality of \eqref{eq:Y-lambda-flow}.

	Because $\dot Y>0$, we may divide the two inequalities and integrate with respect to $Y$. Since $R\asymp\sqrt Y$ on $[0,\tau_K]$,
	\begin{equation}\label{eq:provisional-lambda-Y-growth}
		\frac{d}{dY}\lambda^2=\frac{2\lambda\dot\lambda}{\dot Y}\gtrsim_Ku^2Y^{r-1}, \qquad \lambda(t)^2\geq\lambda(0)^2+c_Ku^2\left(Y(t)^r-Y(0)^r\right).
	\end{equation}
	The initial margin in Assumption~\ref{ass:init}c controls the regime $Y\asymp\varepsilon^2$, while the second term controls subsequent growth. In particular $\lambda\gtrsim_KuR^r$, strengthening the lower bound on $\lambda$ in $\mathcal B_K$. The same inequality and $\dot Y\gtrsim s_\star\lambda$ give the two estimates
	\begin{equation}\label{eq:motion-integrals}
		\int_0^t\Gamma R^{r-1}\,ds\lesssim_KR(t), \qquad \int_0^t\Gamma R^r\,ds\lesssim_KR(t)^2,
	\end{equation}
	which follow by substituting $Y$ as the variable of integration in \eqref{eq:provisional-lambda-Y-growth}; near $Y(0)$ the initial margin regularizes the denominator, and away from it the two integrands reduce respectively to $Y^{-1/2}$ and a constant.

	We now strengthen the provisional bounds in $\mathcal B_K(t)$, beginning with the relation between $\lambda$ and the first-layer gain. Integrating the first-layer flow gives
	\begin{equation*}
		W_1(t)=W_1(0)+d_1(t)\hat a^\top+\int_0^t\mathcal E_1\,ds, \qquad d_1(t)\doteq u\int_0^t s_\star\alpha^{L-1}B_1\,ds,
	\end{equation*}
	where $\mathcal E_1\doteq-E_1$ and the middle term has exactly the fixed row space $\hat a$. By \eqref{eq:flow-H}, the first estimate of \eqref{eq:motion-integrals}, and $\|H_1\|_F\lesssim_KuR^{r-1}$,
	\begin{equation*}
		\left\|\int_0^t\mathcal E_1\,ds\right\|_{\operatorname{op}}\lesssim_K\eta_KR(t).
	\end{equation*}
	Reading $\rho_1=\|W_1\hat a\|$ backwards bounds $\|d_1(t)\|\leq\rho_1+C\varepsilon+C_K\eta_KR$, and hence
	\begin{equation}\label{eq:layer1-op}
		\|W_1\|_{\operatorname{op}}\leq\rho_1+C\varepsilon+C_K\eta_KR.
	\end{equation}

	We next apply Proposition~\ref{prop:imbalance-matrix} to adjacent bottleneck layers. For $1\leq\ell<r$, let $K_\ell\doteq W_{\ell+1}^\top W_{\ell+1}-W_\ell W_\ell^\top$. The exact rank-one parts of the adjacent bottleneck flows cancel:
	\begin{equation*}
		\dot K_\ell=2\operatorname{Sym}\left(W_{\ell+1}^\top\mathcal E_{\ell+1}-\mathcal E_\ell W_\ell^\top\right).
	\end{equation*}
	Using the provisional bounds in $\mathcal B_K$, \eqref{eq:flow-H} and the second estimate of \eqref{eq:motion-integrals},
	\begin{equation}\label{eq:K-bound}
		\|K_\ell(t)\|_{\operatorname{op}}\leq\|K_\ell(0)\|_{\operatorname{op}}+C_K\eta_KR(t)^2\leq2C^2\varepsilon^2+C_K\eta_KR(t)^2.
	\end{equation}
		Weyl's inequality propagates \eqref{eq:layer1-op} through the bottleneck chain and gives the first estimate of \eqref{eq:channel}. For $1\leq\ell<r$, let $\Pi_\ell\doteq W_\ell^\top W_\ell A_{\ell-1}-\rho_\ell^2A_{\ell-1}$. Since $\Pi_\ell\perp A_{\ell-1}$,
	\begin{equation}\label{eq:imbalance-gains}
		\rho_{\ell+1}^2-\rho_\ell^2=a_\ell^\top K_\ell a_\ell+\frac{\|\Pi_\ell\|^2}{\|A_\ell\|^2}\geq-\|K_\ell\|_{\operatorname{op}},
	\end{equation}
	which gives the second estimate of \eqref{eq:channel}. While that estimate is vacuous because $\rho_1\asymp\varepsilon$, monotonicity of $\lambda$ and the deterministic upper bound
	\begin{equation*}
		\lambda=\alpha^{L-1}|A_L|\lesssim u\prod_{i\leq r}\rho_i
	\end{equation*}
	give $\prod_{i\leq r}\rho_i\gtrsim\varepsilon^r$ and dividing by the operator-norm bounds for the other factors proves the third estimate. The two regimes together yield $\rho_\ell\asymp\|W_\ell\|_{\operatorname{op}}\asymp R$ with constants depending only on the fixed problem parameters once $\eta_K$ is sufficiently small. Repeating \eqref{eq:provisional-lambda-Y-growth} with these improved constants proves \eqref{eq:lambda-Y-growth} with the same uniform dependence.

	The preceding bottleneck-layer estimates and the order-one bounds beyond the bottleneck give $\|H_\ell\|_F\lesssim uR^r$ for $\ell>r$, so \eqref{eq:flow-H} and the second estimate of \eqref{eq:motion-integrals} give \eqref{eq:post-bottleneck-displacement}. For $m_0$ small this estimate is strictly stronger than the displacement condition in $\mathcal B_K$ and keeps the singular values of every post-bottleneck hidden matrix, together with the readout norm, in fixed intervals. Since $\|P_{\ell-1}\hat a\|=\alpha^{\ell-1}\|A_{\ell-1}\|/u$ exactly while $\|P_{\ell-1}\|_{\operatorname{op}}\leq\alpha^{\ell-1}\prod_{i<\ell}\|W_i\|_{\operatorname{op}}$, the bottleneck-layer bounds also give $\chi\lesssim1$.

	For $\nu$, split
	\begin{equation*}
		W_i=V_ia_{i-1}^\top+T_i, \qquad V_i\doteq W_ia_{i-1}, \qquad T_i\doteq W_i\left(I-a_{i-1}a_{i-1}^\top\right).
\end{equation*}
	Then $\|V_i\|=\rho_i$ and $\|T_i\|_{\operatorname{op}}\leq\|W_i\|_{\operatorname{op}}\lesssim\rho_i$. Expanding $P_{\ell-1}$ over the resulting rank-one and transverse factors, and using $\|T\Sigma^{1/2}\|_F\leq\|T\|_{\operatorname{op}}\sqrt{\operatorname{tr}\Sigma}$ for every term containing a transverse factor, gives
	\begin{equation*}
		\|P_{\ell-1}\Sigma^{1/2}\|_F\lesssim\alpha^{\ell-1}\left(\prod_{i<\ell}\rho_i\right)\sqrt{\operatorname{tr}\Sigma}.
	\end{equation*}
	Dividing by $\|A_{\ell-1}\|=u\prod_{i<\ell}\rho_i$ proves $\nu\lesssim\sqrt{\operatorname{tr}\Sigma}/u\leq\nu_0$. We have therefore established the announced implication: every scale and ratio condition in \eqref{eq:provisional-bounds} holds with a closing constant $K_\star$ depending only on the fixed problem parameters, while the displacement condition is strict by the preceding paragraph. Since $K>2K_\star$, all inequalities are strict at $\tau_K$, a contradiction. Hence $\mathcal B_K$ holds throughout the small-signal interval.
\end{proof}

We have now proved that the provisional bounds hold throughout the small-signal interval. The next lemma uses them to compare the nonlinear overlap $\gamma$ and its gradients with the linearized quantities $\lambda$ and $H_\ell$.

\begin{lemma}\label{lem:gamma-scale}
	On the small-signal interval, let $J_\ell\doteq\nabla_{W_\ell}\gamma$. Then
	\begin{equation*}
		\gamma=\lambda\left(1+O(\eta_K)\right), \qquad \|J_\ell-H_\ell\|_F\leq C\eta_K\|H_\ell\|_F, \qquad \|H_\ell\|_F\asymp uR^{r-1}\quad(\ell\leq r).
	\end{equation*}
	Moreover $\|B_\ell\|\asymp R^{r-\ell}$ for $\ell\leq r$, where factors contributed by $W_{r+1},\ldots,W_L$ are absorbed into the constants, and when $r<L$ the forward readout alignment satisfies $\rho_L\asymp1$ throughout.
\end{lemma}
\begin{proof}
	The estimates \eqref{eq:forward-backward-errors} and \eqref{eq:backward-error}, applied in the decomposition defining \eqref{eq:Delta-bound} with $w=g_\star$, give $\|J_\ell-H_\ell\|_F\leq C\eta_K\|H_\ell\|_F$ after the channel bounds replace the coarse constants. The same forward expansion gives $|\gamma-\lambda|\leq C\eta_KuR^r$, while \eqref{eq:lambda-Y-growth} and the initial margin give $\lambda\asymp uR^r$, proving the first claim. For the backward tensors, the operator-norm bounds give the upper estimate and
	\begin{equation*}
		\|B_\ell\|\,\|A_\ell\|\geq\left|\langle B_\ell,A_\ell\rangle\right|=\alpha^{-(L-1)}\lambda
	\end{equation*}
	gives the matching lower estimate. The scale of $H_\ell$ follows. Finally, if $r<L$, then $\|A_{L-1}\|\lesssim uR^r$ by the operator-norm bounds for $W_{r+1},\ldots,W_{L-1}$, and
	\begin{equation*}
		\rho_L=\frac{|A_L|}{\|A_{L-1}\|}=\frac{\alpha^{-(L-1)}\lambda}{\|A_{L-1}\|}\gtrsim1;
	\end{equation*}
	the readout norm supplies the upper bound. Thus readout alignment follows from the coupled growth of $\lambda$ and $Y$, without any assumption on the trajectory.
\end{proof}

Lemma~\ref{lem:gamma-scale} compares the nonlinear overlap and its gradient with their linearized approximations at each time. We may now use these comparisons to obtain the growth rate of the nonlinear overlap along gradient flow.

\begin{lemma}\label{lem:gammadot-scale}
	On the small-signal interval, $\dot\gamma\asymp s_\star\mathcal S$ and $\gamma\asymp\lambda\asymp uR^r$.
\end{lemma}
\begin{proof}
	By Lemma~\ref{lem:gamma-scale}, $J_\ell=H_\ell+O(\eta_K\|H_\ell\|_F)$, while Lemma~\ref{lem:grad-scale} gives that $\nabla_{W_\ell}\mathcal L=-s_\star H_\ell+O(\eta_Ks_\star\|H_\ell\|_F)$. Contracting these two expansions in $\dot\gamma=-\sum_\ell\langle J_\ell,\nabla_{W_\ell}\mathcal L\rangle_F$ proves the result.
\end{proof}

The preceding estimates show that the bottleneck gains remain comparable throughout the small-signal interval and become asymptotically equal once they grow beyond their initialization scale.

\begin{corollary}\label{cor:balance}
	On the small-signal interval, $\rho_\ell\asymp X\asymp\rho_1$ for every $\ell\leq r$, and $\rho_\ell=X(1+o(1))$ once $\rho_1/\varepsilon\to\infty$, in particular at escape. The local errors satisfy
	\begin{equation}\label{eq:increment}
		\chi\lesssim1, \qquad \nu\lesssim\nu_0, \qquad \int_0^t\|\dot W_\ell\|_F\,ds\lesssim R(t) \quad(\ell\leq r).
	\end{equation}
\end{corollary}
\begin{proof}
	The comparability and the first two estimates of \eqref{eq:increment} follow directly from Lemma~\ref{lem:channel}. In the subsequent local limit, \eqref{eq:K-bound} is $O(\varepsilon^2)+o(\rho_1^2)$, so \eqref{eq:imbalance-gains} and the matching operator-norm propagation give $\rho_\ell=\rho_1(1+o(1))$ whenever $\rho_1/\varepsilon\to\infty$. Taking the geometric mean gives the sharp statement. The last estimate of \eqref{eq:increment} follows from \eqref{eq:flow-H}, $\|H_\ell\|_F\lesssim uR^{r-1}$ and the first integral in \eqref{eq:motion-integrals}.
\end{proof}

Equipped with the preceding lemmas, Theorem~\ref{thm:critical-depth-escape-law} follows almost immediately.

\begin{proof}[Proof of Theorem~\ref{thm:critical-depth-escape-law}]
	The equality \eqref{eq:lambda-Y-growth} is the improved form of \eqref{eq:provisional-lambda-Y-growth} obtained in Lemma~\ref{lem:channel}. Together with $\dot Y\gtrsim s_\star\lambda$ and $Y(0)\asymp\varepsilon^2$, it gives
	\begin{equation}\label{eq:Y-escape-integral}
		\dot Y\gtrsim\Gamma\sqrt{Y^r-Y(0)^r+c\varepsilon^{2r}}.
	\end{equation}
	Integrating from $Y(0)$ to a fixed multiple of $\vartheta^2$ and rescaling $Y=\varepsilon^2z$ shows
	\begin{equation*}
		\tau\lesssim\frac{1}{\Gamma}\int_{Y(0)}^{C\vartheta^2}\frac{dY}{\sqrt{Y^r-Y(0)^r+c\varepsilon^{2r}}}\asymp\begin{cases}
			\Gamma^{-1}, & r=1,\\
			\Gamma^{-1}\log\left(1/\varepsilon\right), & r=2,\\
			\Gamma^{-1}\varepsilon^{-(r-2)}, & r\geq3.
		\end{cases}
	\end{equation*}
	By Corollary~\ref{cor:balance}, reaching this level of $Y$ is equivalent, up to fixed constants, to reaching $X=\vartheta$; changing a fixed threshold changes only the implicit constant.

	For the lower bound, Lemma~\ref{lem:grad-scale} and the bottleneck-layer estimates of Lemma~\ref{lem:channel} give $\|\dot W_\ell\|_F\lesssim\Gamma M^{r-1}$ for every $\ell\leq r$. The maximum $M$ is locally Lipschitz and therefore satisfies $\dot M\lesssim\Gamma M^{r-1}$ almost everywhere. Integrating this scalar inequality from $M(0)\asymp\varepsilon$ to a fixed threshold gives the matching constant, logarithmic and power-law lower bounds for $r=1$, $r=2$ and $r\geq3$, respectively. The threshold relation in the statement ensures that neither $M$ nor $\zeta$ reaches $m_0$ before escape, so both lemmas apply throughout. Finally, the readout conclusion is Lemma~\ref{lem:gamma-scale} evaluated at escape.
\end{proof}

Lemma~\ref{lem:channel} shows that every layer $W_\ell$ with $\ell>r$ remains close enough to its initialization to preserve the bounds in Assumption~\ref{ass:init}b throughout the small-signal interval. Restricting its displacement estimate to the escape interval gives the sharper order-$\vartheta^2$ bound recorded next.

\begin{lemma}\label{lem:persistence}
	On $[0,\tau]$, every layer with $\ell > r$ satisfies $\|W_\ell(t) - W_\ell(0)\|_F \lesssim \vartheta^2$.
\end{lemma}
\begin{proof}
	The last step of the proof of Lemma~\ref{lem:channel} gives $\int_0^t\|\dot W_\ell\|_F\,ds \lesssim R(t)^2$ for $\ell > r$, and on $[0,\tau]$ Corollary~\ref{cor:balance} gives $R\lesssim\varepsilon+X\lesssim\vartheta$. The bound is independent of $\Gamma$, of $u$ and of $\varepsilon$.
\end{proof}

The same computation run at a bottleneck layer gives a qualitatively different result. For $\ell\leq r$ the gain is $\rho_\ell\asymp X$ rather than $\Theta(1)$, so $\|J_\ell\|_F\asymp uX^{r-1}\asymp\gamma/X$ (larger by one factor of $X$) and \eqref{eq:increment} returns a displacement of order $\vartheta$ rather than $\vartheta^2$. The absent factor of $X$ is quite important: at escape the bottleneck-layer norms are themselves of order $\vartheta$, so an order-$\vartheta$ displacement is the same size as the layer itself, not a small correction to it. Where Lemma~\ref{lem:persistence}'s order-$\vartheta^2$ bound lets the layers with $\ell>r$ be treated as fixed near their initialization throughout, no such approximation is available for $\ell\leq r$: these layers move by an amount comparable to their final scale. Indeed, this is why our proof tracks $Y$ (equivalently $\lambda$) rather than treating $W_1,\ldots,W_r$ as a small perturbation of their initial values.

Theorem~\ref{thm:critical-depth-escape-law} predicts \emph{when} the network leaves the saddle, but not \emph{what} it learns as it does so. This distinction is particularly important for a multimode teacher: Section~\ref{sec:critical-depth} shows that all modes enter the escape time through the single aggregate signal $\Gamma=\|\E[yx]\|$, but this does not by itself provide information about whether the modes emerge separately or together. In the next section, we show that they emerge together during the first escape. The network grows along a single direction, and the individual mode overlaps therefore share a common amplitude with relative weights determined by the teacher and input covariance.

\section{Joint Linearized Escape}\label{sec:multi-mode}

Section~\ref{sec:critical-depth} established that the escape time depends on a multimode teacher through the single vector $w_\star$. We now characterize the corresponding dynamics. During the first small-signal escape, the network's linearized predictor grows primarily along the presented direction $\Sigma w_\star$, so the teacher modes drive one joint departure from the saddle rather than separate leading-order departures. In this section, we show that the modewise linear overlaps $\E[fg_k]$ grow jointly up to the escape time $\tau$.\footnote{Our theory does not address how or when the network subsequently learns the nonlinear residuals $\sigma_\perp(g_k)$, although this is an interesting avenue for future work.}

Write $p \doteq \alpha^{L-1}(W_L\cdots W_1)^\top$ for the linearized predictor and $\hat a \doteq \Sigma w_\star/\|\Sigma w_\star\|$ for the presented direction, so that $A_0=u\hat a$. The argument will require three results from Section~\ref{sec:critical-depth}. Recall that Lemma~\ref{lem:grad-scale} identifies the leading rank-one term in each layer gradient, Lemma~\ref{lem:channel} keeps the bottleneck-layer gains comparable throughout the small-signal interval, and Lemma~\ref{lem:gamma-scale} transfers these linearized estimates to the nonlinear network.

\subsection{Rank-One Drive}

Every layer receives the same leading rank-one drive, anchored at the vector $A_0$. Propagating these drives through the product gives a closed expression for the motion of $p$.

\begin{lemma}\label{lem:rank-one-drive}
	Let $Q_j \doteq (W_j\cdots W_1)^\top(W_j\cdots W_1)$ with $Q_0 \doteq I$, and let $\chi$ be as in \eqref{eq:nu-chi}. Then, on the small-signal interval,
	\begin{equation}\label{eq:drive}
		\dot p = s_\star\alpha^{2(L-1)}\sum_{\ell=1}^{L}\|B_\ell\|^2\,Q_{\ell-1}A_0 + E, \qquad \|E\| \leq C\eta_K\chi s_\star\alpha^{2(L-1)}\sum_{\ell=1}^{L}\|B_\ell\|^2\left\|Q_{\ell-1}A_0\right\| .
	\end{equation}
\end{lemma}
\begin{proof}
	\emergencystretch=1em
	By the product rule, $\dot p^\top = \alpha^{L-1}\sum_\ell (W_L\cdots W_{\ell+1})\dot W_\ell(W_{\ell-1}\cdots W_1)$, into which we insert $\dot W_\ell = s_\star\alpha^{L-1}B_\ell A_{\ell-1}^\top - E_\ell$ from Lemmas~\ref{lem:gamma-scale} and~\ref{lem:grad-scale}. Note that $B_\ell = (W_L\cdots W_{\ell+1})^\top$ by definition, so $(W_L\cdots W_{\ell+1})B_\ell = \|B_\ell\|^2$, which is a scalar. Also, $A_{\ell-1} = (W_{\ell-1}\cdots W_1)A_0$, so $A_{\ell-1}^\top(W_{\ell-1}\cdots W_1) = A_0^\top Q_{\ell-1}$.

	For the error, $\|W_L\cdots W_{\ell+1}\|_{\operatorname{op}} = \|B_\ell\|$, so layer $\ell$ contributes at most
	\begin{equation*}
		\alpha^{L-1}\|B_\ell\|\|E_\ell\|_{\operatorname{op}}\|W_{\ell-1}\cdots W_1\|_{\operatorname{op}} .
	\end{equation*}
	From Lemma~\ref{lem:grad-scale}, this is at most
	\begin{equation*}
		C\eta_Ks_\star\alpha^{2(L-1)}\|B_\ell\|^2\|A_{\ell-1}\|\,\|W_{\ell-1}\cdots W_1\|_{\operatorname{op}} .
	\end{equation*}
	Since $Q_{\ell-1}$ is positive semidefinite and $\|A_{\ell-1}\| = u\|(W_{\ell-1}\cdots W_1)\hat a\|$,
	\begin{equation*}
		\|Q_{\ell-1}A_0\| \geq \frac{A_0^\top Q_{\ell-1}A_0}{\|A_0\|} = \frac{\|A_{\ell-1}\|^2}{u} .
	\end{equation*}
	Thus the ratio of layer $\ell$'s error to its main term is at most $C\eta_K\chi$.
\end{proof}

The factor $\chi$ is the maximum of the layerwise ratios introduced in \eqref{eq:nu-chi} and bounded in Lemma \ref{lem:channel}. For the term $\ell=1$, the corresponding layerwise ratio equals one identically because $P_0=I$. Before using it, let us make two remarks on the main term.

The first is that Lemma~\ref{lem:rank-one-drive} contains the exact multilinear version of Lemma~\ref{lem:gammadot-scale}. Since $\lambda=\langle p,A_0\rangle$ exactly and $A_0^\top Q_{\ell-1}A_0 = \|A_{\ell-1}\|^2$, projecting \eqref{eq:drive} onto $A_0$ returns
\begin{equation*}
	\dot\lambda=s_\star\mathcal S\left(1+O\left(\eta_K\right)\right).
\end{equation*}
The corresponding nonlinear statement is Lemma~\ref{lem:gammadot-scale}. The scalar theory of Section~\ref{sec:critical-depth} is the projection of the flow of $p$ onto the single axis $A_0$.

The second is that the number of teacher modes has disappeared. Whatever $r^*$ is, the leading term of \eqref{eq:drive} is built from one vector; each $Q_{\ell-1}$ is positive semidefinite, so that term lies in the convex cone generated by the images $Q_{\ell-1}A_0$, and $\langle A_0,\dot p\rangle=s_\star\mathcal S\left(1+O\left(\eta_K\right)\right)>0$. Intuitively, the network cannot grow $r^*$ independent channels at leading order because it is not driven along $r^*$ independent directions.

\subsection{Alignment and Mode Ratios}

The cone in Lemma~\ref{lem:rank-one-drive} need not collapse to a single ray, and initially its transverse and parallel components may be of comparable size. However, the first-layer increment has fixed row space $\hat a$, while the only source of transverse motion is the initial first-layer weight (whose size remains $O(\varepsilon)$). Thus, as the signal channel grows beyond its initialization scale, the transverse component becomes negligible relative to the component along $\hat a$.

\begin{proposition}\label{prop:alignment}
	On the small-signal interval,
	\begin{equation*}
		\frac{\left\|p - \langle p,\hat a\rangle \hat a\right\|}{\|p\|} = O\left(\frac{\varepsilon}{\rho_1} + \eta_K\right),
	\end{equation*}
	and hence, by Corollary~\ref{cor:balance}, at escape
	\begin{equation*}
		\frac{\left\|p - \langle p,\hat a\rangle \hat a\right\|}{\|p\|} = O\left(\varsigma\right), \qquad \frac{\langle p,\hat a\rangle}{\|p\|} = 1 - O\left(\varsigma^2\right), \qquad \varsigma \doteq \frac{\varepsilon}{\vartheta} + \eta_K .
	\end{equation*}
\end{proposition}
\begin{proof}
		Integrating the first-layer flow and setting $d_1(t)\doteq u\int_0^t s_\star\alpha^{L-1}B_1\,ds$,
		\begin{equation*}
			W_1(t) = W_1(0) + d_1(t)\hat a^\top + \int_0^t \mathcal E_1\,ds .
		\end{equation*}
		The middle term is exactly rank one with row space $\hat a$. Writing $\tilde{\mathcal E}_1 \doteq \int_0^t\mathcal E_1\,ds$, we therefore have
		\begin{equation*}
			p^\top = \alpha^{L-1}(W_L\cdots W_2)\left[d_1(t)\hat a^\top + W_1(0) + \tilde{\mathcal E}_1\right],
		\end{equation*}
	and only $W_1(0) + \tilde{\mathcal E}_1$ can contribute off $\hat a$: every later factor multiplies on the left and so changes the coefficient, not the row space. Using $\|W_L\cdots W_2\|_{\operatorname{op}} = \|B_1\|$,
	\begin{equation*}
		\left\|p - \langle p,\hat a\rangle\hat a\right\| \leq \alpha^{L-1}\|B_1\|\left(\|W_1(0)\|_{\operatorname{op}} + \|\tilde{\mathcal E}_1\|_{\operatorname{op}}\right).
	\end{equation*}
	For the denominator, we need not estimate the parallel component of the product directly, we can use that $\lambda=u\langle p,\hat a\rangle$ gives $\|p\|\geq\lambda/u$, and Lemma~\ref{lem:gamma-scale} gives $\lambda\asymp u\prod_{i=1}^r\rho_i$ and $\|B_1\|\asymp\prod_{i=2}^r\rho_i$, so $\|p\|\gtrsim\prod_{i=1}^r\rho_i$. Assumption~\ref{ass:init}a bounds $\|W_1(0)\|_{\operatorname{op}}\leq C\varepsilon$, while
	\begin{equation*}
		\|\tilde{\mathcal E}_1\|_{\operatorname{op}} \leq \int_0^t\|\mathcal E_1\|_F\,ds \lesssim \eta_K\int_0^t s_\star\|H_1\|_F\,ds \lesssim \eta_KR
	\end{equation*}
	by Lemma~\ref{lem:grad-scale} and \eqref{eq:increment}. Dividing,
	\begin{equation*}
		\frac{\left\|p - \langle p,\hat a\rangle\hat a\right\|}{\|p\|} \lesssim \frac{\left(\prod_{i=2}^r\rho_i\right)\left(\varepsilon + \eta_KR\right)}{\prod_{i=1}^r\rho_i} \lesssim \frac{\varepsilon}{\rho_1} + \eta_K,
	\end{equation*}
	which is the first claim. At escape $X = \vartheta$, so Corollary~\ref{cor:balance} gives $\rho_1 \asymp \vartheta$. Finally, writing $\omega$ for the angle between $p$ and $\hat a$, the bound proved reads $|\sin\omega| = O(\varsigma)$, and $\cos\omega = \sqrt{1-\sin^2\omega}$ gives $1 - O(\sin^2\omega)$.
\end{proof}

The bound $O\left(\varepsilon/\rho_1+\eta_K\right)$ holds throughout the small-signal interval. The first term is the initial component transverse to $\hat a$, whose relative contribution decreases as the first-layer gain $\rho_1$ grows. The second controls the accumulated deviation of the gradient from its leading rank-one term and vanishes in the local limit by \eqref{eq:eta-K}. At escape, Corollary~\ref{cor:balance} gives $\rho_1\asymp\vartheta$, reducing the bound to $O\left(\varepsilon/\vartheta+\eta_K\right)=O(\varsigma)$.

The observables of the several modes are now immediate. Writing $\gamma_k \doteq \E[fg_k]$ for the overlap of the network with mode $k$, Assumption~\ref{ass:input}a gives $\E[xg_k] = \Sigma v_k/s_k$, so that the forward estimate in Lemma~\ref{lem:gamma-scale} gives
\begin{equation}\label{eq:gamma-k}
	\gamma_k=u_k\left(\left\langle p,\hat a_k\right\rangle+e_k\right), \qquad |e_k|\leq C\eta_K\|p\|,
\end{equation}
with $u_k$ and $\hat a_k$ as in Section~\ref{subsec:rate-constant}. Up to the common controlled error, all $r^*$ observables are linear functionals of one vector, and by Proposition~\ref{prop:alignment} that vector is $\hat a$ up to a relative $O(\varsigma)$.

\begin{corollary}\label{cor:mode-ratios}
	Let $c_\star \doteq \min_{k\in[r^*]}\left|\langle \hat a,\hat a_k\rangle\right|$ and suppose $c_\star > 0$. Then, at escape
	\begin{equation*}
		\frac{\gamma_k}{\gamma_j} = \frac{u_k\langle\hat a,\hat a_k\rangle}{u_j\langle\hat a,\hat a_j\rangle}\left(1 + O\left(\frac{\varsigma}{c_\star}\right)\right),
	\end{equation*}
	so that the ratios are determined only by $\Sigma$ and the teacher, and are independent of the initialization.
\end{corollary}
\begin{proof}
	Write $p=\|p\|\left(\cos\omega\,\hat a+\sin\omega\,\hat e\right)$ with $\hat e\perp\hat a$ and $|\sin\omega|=O(\varsigma)$ by Proposition~\ref{prop:alignment}. Then \eqref{eq:gamma-k} gives $\gamma_k=u_k\|p\|\langle\hat a,\hat a_k\rangle\left(1+O\left(\varsigma/c_\star\right)\right)$, where the $e_k$ term is absorbed because $\eta_K\leq\varsigma$, and the common factor $\|p\|$ cancels in the ratio.
\end{proof}

Corollary~\ref{cor:mode-ratios} has a direct Gram-matrix form. Defining the covariance-image Gram matrix
\begin{equation}\label{eq:mode-Gram}
	\mathsf G_{kj} \doteq u_ku_j\left\langle \hat a_k,\hat a_j\right\rangle = \left\langle \frac{\Sigma v_k}{s_k},\frac{\Sigma v_j}{s_j}\right\rangle,
\end{equation}
$\hat a$ is the normalized direction of $\sum_{j=1}^{r^*}\beta_ju_j\hat a_j$, and we have $u_k\langle\hat a,\hat a_k\rangle\propto\left(\mathsf G\beta\right)_k$, with the same positive proportionality constant for every $k$. Consequently, when the entries of $\mathsf G\beta$ are nonzero, Corollary~\ref{cor:mode-ratios} becomes
\begin{equation}\label{eq:mode-ratios-Gram}
	\frac{\gamma_k}{\gamma_j} = \frac{\left(\mathsf G\beta\right)_k}{\left(\mathsf G\beta\right)_j}\left(1+O\left(\frac{\varsigma}{c_\star}\right)\right).
\end{equation}
For isotropic inputs and unit mode directions, $\mathsf G_{kj}=\langle v_k,v_j\rangle$, so Corollary~\ref{cor:mode-ratios} gives mode-overlap ratios proportional to the entries of $\mathsf G\beta$. When the mode directions are orthonormal, $\mathsf G=I$ and this reduces to $\gamma_k/\gamma_j=\beta_k/\beta_j$ up to the stated error. For nonorthogonal modes, each entry of $\mathsf G\beta$ contains cross-terms from the other amplitudes, so the observed overlaps are mixtures, though they arise from the same leading linearized predictor.

Once $\rho_1\gg\varepsilon$, every mode overlap $\gamma_k$ is a fixed multiple of the common amplitude $\|p\|$, up to the alignment error in Proposition~\ref{prop:alignment}. The overlaps therefore grow together during the controlled small-signal interval, and their ratios at escape are given by Corollary~\ref{cor:mode-ratios}. Any variation in these ratios within the interval is limited by the $O\left(\varsigma\right)$ rotation of $p$ away from $\hat a$.

Figure~\ref{fig:multimode} shows this joint evolution. In panel (a), the observed Mode 1 trajectory supplies the common amplitude, and \eqref{eq:mode-ratios-Gram} determines the other two trajectories through
\begin{equation*}
	\widehat\gamma_k(t) \doteq \gamma_1(t)\frac{\left(\mathsf G\beta\right)_k}{\left(\mathsf G\beta\right)_1}.
\end{equation*}
Neither Mode 2 nor Mode 3 is fitted or normalized by its terminal value. Panel (b) evaluates the absolute mode ratios at escape across teacher geometries and initialization scales. Each point places the Gram prediction $\left|\left(\mathsf G\beta\right)_k/\left(\mathsf G\beta\right)_1\right|$ on the horizontal axis and the observed ratio $|\gamma_k/\gamma_1|$ on the vertical axis. The nonorthogonal and partially cancelling cases remain close to this diagonal; i.e. $\mathsf G\beta$ captures the covariance-induced mixing between modes.

\begin{figure}[t]
	\centering
	\includegraphics[width=\linewidth]{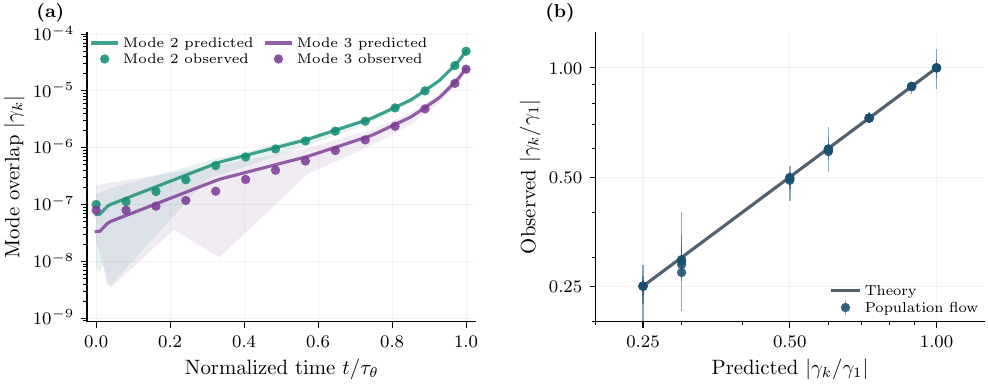}
	\caption{\textbf{Mode overlaps share one linearized escape trajectory.} \textbf{(a)} Mode 2 and Mode 3 population-flow overlaps (dots and bands) and their predictions from the observed Mode 1 trajectory and the Gram ratios (lines). Neither predicted trajectory is fitted to its corresponding mode. \textbf{(b)} Predicted versus observed absolute mode ratios at escape across teacher geometries and initialization scales. Predictions are $\left|\left(\mathsf G\beta\right)_k/\left(\mathsf G\beta\right)_1\right|$. Dots report medians, and bands and error bars span ten independent initializations.}
	\label{fig:multimode}
\end{figure}

\section{Discussion and Conclusion}\label{sec:discussion}

\subsection{Summary of Contributions}

Our analysis gives a simple understanding of the first escape from the saddle in deep nonlinear networks with small initialization. We first derived a nonlinear analog of the balance law used in the analysis of deep linear networks, showing that although smooth nonhomogeneous activations do not preserve the exact balance, the deviation from exact balance is high-order near the origin. This approximate balance persists through the first escape and reduces the full gradient flow to the growth of a single signal channel. Then, when the first $r$ layers are initialized at scale $\varepsilon$ and the remaining layers at scale one, the critical depth $r$, not the total depth $L$, sets the scaling of the escape time. The resulting delay is constant for $r=1$, logarithmic for $r=2$, and polynomial in $\varepsilon^{-1}$ for $r\geq3$.

The escape time is also determined by the signal presented to the network near the saddle. This signal is captured by $\Gamma = \|\E[yx]\|$, which reduces the relevant dependence on the target and input distributions to a single vector moment. For a multimode teacher, we further show that $\Gamma$ is the norm of a sum of the image of mode contributions under the covariance. Thus, anisotropy of the input distribution, mode orientation, and cancellation affect the escape jointly rather than only through independent mode strengths. Importantly, $\Gamma$ is directly computable from the dataset.

Finally, we show that the same reduction determines what the network learns during its escape, not just when it escapes. With scalar output, the linearized predictor grows along the direction $\Sigma w_\star$, rather than along a separate direction per teacher mode. The individual mode overlaps therefore grow through a common amplitude with relative weights determined by the covariance. Thus, even when a teacher contains several nonlinear modes, the first small-signal departure is a joint escape rather than a sequence of modewise escapes.

\subsection{Scope and Limitations}

Our theory controls the first escape from a small-signal neighborhood. It does not describe the subsequent fitting of the nonlinear residual $y_\perp$, including the mode-specific components $\sigma_\perp(g_k)$. At escape, Lemma~\ref{lem:channel} and Corollary~\ref{cor:balance} give $\|W_\ell\|_{\operatorname{op}}\asymp\vartheta$ for $\ell\leq r$, whereas Lemma~\ref{lem:persistence} keeps the layers with $\ell>r$ close to initialization, so the same small-signal analysis cannot be directly reapplied.

Importantly, the operator-norm bounds themselves do not control every singular direction. The leading updates grow primarily along one rank-one path, and transverse directions may remain much smaller than $\vartheta$. Such directions can subsequently be used by $y_\perp$, potentially producing later delays or modewise learning that are absent from the first-escape dynamics. Our no-staging statement is therefore limited to the linear overlaps on $[0,\tau]$.

This rank-one structure also distinguishes the present setting from classical deep-linear analyses, where modewise learning is a rank phenomenon rather than a nonlinearity phenomenon: matrix-valued input-output correlations with several singular directions produce different growth times \citep{saxe2014exactsolutionsnonlineardynamics}, and related small-initialization saddle-to-saddle dynamics have been established in diagonal linear networks \citep{PesmeFlammarion2023}. Here the output is scalar, so $\E[yx]=\Sigma w_\star$ is one vector with one leading direction, however many nonlinear modes the teacher contains. A scalar-output deep linear problem has the same rank-one property, and the $r^*$ teacher modes in our setting are multiple contributions to that one linearized direction, not $r^*$ separate singular directions.

\subsection{Open Questions}

The immediate question is what replaces the rank-one escape dynamics after the network leaves the small-signal interval. The nonlinear residual can then drive transverse directions that were invisible to $\E[yx]$. Determining whether those directions undergo additional critical-depth delays, and whether their recruitment produces genuine modewise staging, would connect the theory presented here to the longer feature-learning trajectory.

Vector-valued outputs are another natural extension. In that setting the input-output correlation is a matrix with potentially several singular directions, so critical-depth escape and staging may coexist.

Further open questions concern optimization and scaling. Extending the escape law to finite-step gradient descent and stochastic gradients would quantify when discretization and noise alter the predicted timescale, while another question is whether nonsmooth activations admit an analogous theory once the local Taylor expansion used here is no longer available. Finally, identifying exact prefactors rather than rate classes would separate universal features of the escape from architecture- and activation-dependent constants.

Small initialization therefore does not impose a uniform delay on deep nonlinear networks. The number of small layers determines the escape exponent, the input-output correlation determines the signal timescale, and the rank of the presented signal determines how many leading directions can grow. These three quantities provide a compact description of the network's first departure from the saddle and a starting point for understanding the nonlinear stages that follow.

\section*{Acknowledgements and Funding Disclosure}
The authors thank David Kim for many useful discussions and Arjun Banerjee for comments on the manuscript. This work was supported in part by the U.S. Army Research Laboratory and the U.S. Army Research Office under Contract No. W911NF-201-0151.

\appendix
\section{Additional Experiments}\label{app:additional-experiments}

All code to reproduce the empirical results presented is available on GitHub: \url{https://github.com/divitr/saddle-escape}.

\subsection{Alternative Signal Coordinates}\label{app:signal-ablation}

Writing $\Gamma_k\doteq h_\sigma|\beta_k|u_k$ for the length of the $k$th contribution, we compare five candidate coordinates $P$ in the prediction $\tau_\vartheta=C/P$: the vector norm $\Gamma$, the teacher scale $s_\star$, and three angle-blind combinations of the $\Gamma_k$. For each, $C$ is fixed by the median log residual on $17$ points and evaluated on the one held out.

\begin{table}[h]
	\centering
	\caption{\textbf{Ablation of the signal coordinate.} The final column gives $\exp\left(\operatorname{error}\right)-1$ as a multiplicative discrepancy.}
	\label{tab:signal-ablation}
	\begin{tabular}{lcc}
		\toprule
		Signal coordinate $P$ & Mean absolute log error & Multiplicative error \\
		\midrule
		$\Gamma$ & $0.0081$ & $0.81\%$ \\
		$s_\star$ & $0.0944$ & $9.90\%$ \\
		$\left(\sum_k\Gamma_k^2\right)^{1/2}$ & $0.3194$ & $37.63\%$ \\
		$\max_k\Gamma_k$ & $0.3208$ & $37.82\%$ \\
		$\sum_k\Gamma_k$ & $0.4004$ & $49.25\%$ \\
		\bottomrule
	\end{tabular}
\end{table}

\subsection{Robustness Across Activations}\label{app:activation-robustness}

Our primary experiments all use $\tanh$ for convenience; here we show that Theorem~\ref{thm:critical-depth-escape-law} holds on other activations satisfying the conditions we require. Table~\ref{tab:activation-robustness} repeats the critical-depth experiment at $r=4$ for $\tanh$, GELU and SiLU, and reports the effective constant $C(\varepsilon)\doteq\tau_\vartheta/\left(\varepsilon^{-2}-\vartheta^{-2}\right)$.

\begin{table}[h]
	\centering
	\caption{\textbf{Effective constant converges across activations.} For $r=4$ the escape law predicts $\tau_\vartheta = C\left(\varepsilon^{-2}-\vartheta^{-2}\right)$ with $C$ constant. Entries are the measured $C(\varepsilon)\doteq\tau_\vartheta/\left(\varepsilon^{-2}-\vartheta^{-2}\right)$, normalized within each activation by its value at the smallest scale.}
	\label{tab:activation-robustness}
	\begin{tabular}{cccc}
		\toprule
		$\varepsilon/\vartheta$ & $\tanh$ & GELU & SiLU \\
		\midrule
		$0.12$ & $1.000$ & $1.000$ & $1.000$ \\
		$0.18$ & $1.003$ & $1.004$ & $1.003$ \\
		$0.27$ & $1.011$ & $1.011$ & $1.012$ \\
		$0.40$ & $1.031$ & $1.030$ & $1.032$ \\
		$0.60$ & $1.095$ & $1.094$ & $1.098$ \\
		\bottomrule
	\end{tabular}
\end{table}

\subsection{Realistic Training Conditions}\label{app:real-robustness}

Figure~\ref{fig:real-robustness} tests the predictions of Theorem~\ref{thm:critical-depth-escape-law} on different escape thresholds and on discrete full-batch gradient descent as well as on minibatch SGD.

\begin{figure}[t]
	\centering
	\includegraphics[width=\linewidth]{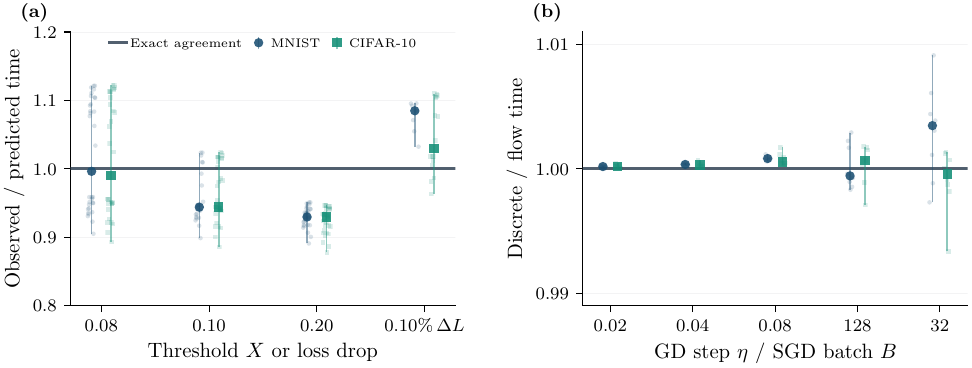}
	\caption{\textbf{The real-data prediction survives changes to its endpoint and training dynamics.} Points show individual initialization ratios; markers show medians and bars span initialization replicates. The horizontal line is exact agreement. \textbf{(a)} Target-aligned thresholds and a training-loss endpoint that does not use $w_\star$ or $X$. \textbf{(b)} Full-batch gradient descent and minibatch SGD against matched empirical gradient flow.}
	\label{fig:real-robustness}
\end{figure}

\bibliographystyle{plainnat}
\bibliography{references}

\end{document}